\documentclass[11pt]{article}
\usepackage[margin=1in]{geometry}
\usepackage{times}
\usepackage[round,authoryear]{natbib}
\setcitestyle{citesep={;},aysep={,},yysep={;}}
\usepackage{microtype}

\usepackage{amsmath,amssymb,amsthm}
\usepackage{graphicx}
\usepackage{booktabs}
\usepackage{array}
\usepackage{xcolor}
\newcommand{\refrev}[1]{#1}
\newcommand{\citrev}[1]{#1}
\usepackage{float}
\usepackage{stfloats}
\fnbelowfloat
\renewcommand{\bottomfraction}{0.5}
\usepackage{hyperref}
\hypersetup{
  colorlinks=true,
  linkcolor=blue!50!black,
  citecolor=blue!50!black,
  urlcolor=blue!50!black,
  pdftitle={What FID Hides: Detecting, Ranking, and Diagnosing Deviations in Generative Evaluation},
  pdfauthor={Hao Chen}
}
\usepackage{url}
\usepackage{placeins}

\newtheorem{proposition}{Proposition}

\title{What FID Hides: Detecting, Ranking, and\\
Diagnosing Deviations in Generative Evaluation}
\author{Hao Chen \\ Department of Statistics \\ University of California, Davis \\ \texttt{hxchen@ucdavis.edu}}
\date{}

\begin{document}
\raggedbottom  
\maketitle

\begin{abstract}
Generative models are commonly ranked by Fr\'echet Inception Distance (FID) and Kernel Inception
Distance (KID), yet FID's first-two-moment summary can miss distributional
differences, and a reported scalar gap alone is not a calibrated test against sampling variation.
FID's moment restriction has concrete consequences: on ImageNet,
visually unrecognizable images optimized only to match the reference Inception mean and covariance
obtain FID $24.7$ versus $58.6$ for held-out real images (lower is better).
Moreover, FID and KID are scalar discrepancies that are unchanged when the two samples are exchanged and therefore do not encode the direction of a
dispersion change: under-dispersion, as can occur in mode collapse, versus over-dispersion.
We introduce \textbf{ZID} (\emph{Z-resolved Integrated Diagnostic}), which combines six standardized
location- and dispersion-sensitive arms from a rank graph (RISE) and Gaussian kernels (GPK at two
bandwidths). Rather than asking one scalar to serve incompatible roles,
ZID reports three linked outputs: an index for ranking departure magnitude,
a permutation $p$-value for testing distributional equality, and a signed dispersion readout for
diagnosis. In controlled experiments, ZID detects a broad range of departures, and its score tracks increasing severity along the corresponding sweeps, including cases in which FID is flat or reversed.
On DiT-XL/2 and SiT-XL/2 guidance sweeps, ZID detects
departure from real data, and its signed readout labels the high-guidance diversity
collapse as under-dispersion.
\end{abstract}

\section{Introduction}\label{sec:intro}

Generative models are increasingly used across science and industry, making
model-selection metrics consequential: a misleading ranking can favor a worse model. Two widely used
embedding-based scalar summaries are FID and KID. FID \citep{heusel2017fid} fits a Gaussian to each
set of embedded samples and reports their Fr\'echet distance; KID \citep{binkowski2018kid} reports the
maximum mean discrepancy associated with a degree-$3$ polynomial kernel on the same embeddings. Both
are commonly used to rank models.

Because FID depends only on the first two moments of the embedding
distribution, \refrev{optimizing a noise-initialized batch of $64{\times}64$ images to match the
Inception feature mean and covariance of an ImageNet reference}
\citep{deng2009imagenet,russakovsky2015imagenet} yields visual noise to which FID assigns a lower, nominally better value than to held-out
real images: $24.7$ versus the $58.6$ finite-sample real--real baseline
(\refrev{Inception-2048}, $n{=}512$;
Fig.~\ref{fig:pixelgaming} and
Sec.~\ref{sec:pixelgaming}), \refrev{demonstrating FID's matched-moment blind spot}. \refrev{By contrast, the ZID score is much larger for the optimized image set than for held-out real images (Fig.~\ref{fig:pixelgaming}).}

The same failure appears directly in feature space. At full-dimensional
Inception-2048 ($m{=}n{=}2500$, with full-rank sample covariances), FID is $19.5$ for real versus
held-out real features but $0.003$ for real versus a moment-restored bimodal batch;
the corresponding KID estimates are $9.5{\times}10^{-5}$ and
$-3.2{\times}10^{-4}$, respectively (App.~\ref{app:feature-gaming}). Related Gaussian
constructions reproduce the FID inversion after PCA reduction, on
\citrev{Tiny-ImageNet}\footnote{Tiny-ImageNet is the 200-class dataset distributed by
\href{http://cs231n.stanford.edu/tiny-imagenet-200.zip}{Stanford CS231n}.}, and in
\citrev{DINO features \citep{caron2021dino}}
(App.~\ref{app:feature-gaming}).

\begin{figure}[!t]\centering
\includegraphics[width=0.4\textwidth]{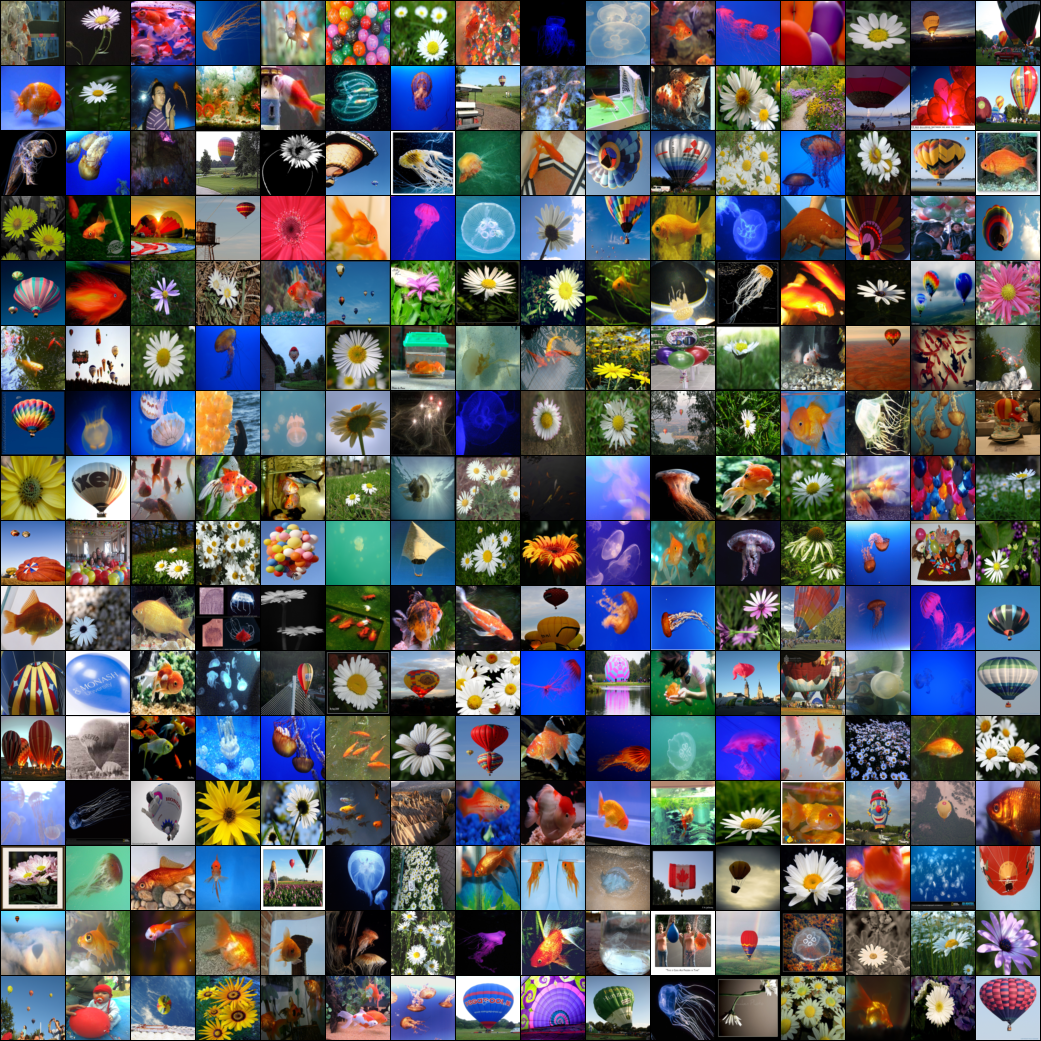}\hspace{11mm}%
\includegraphics[width=0.4\textwidth]{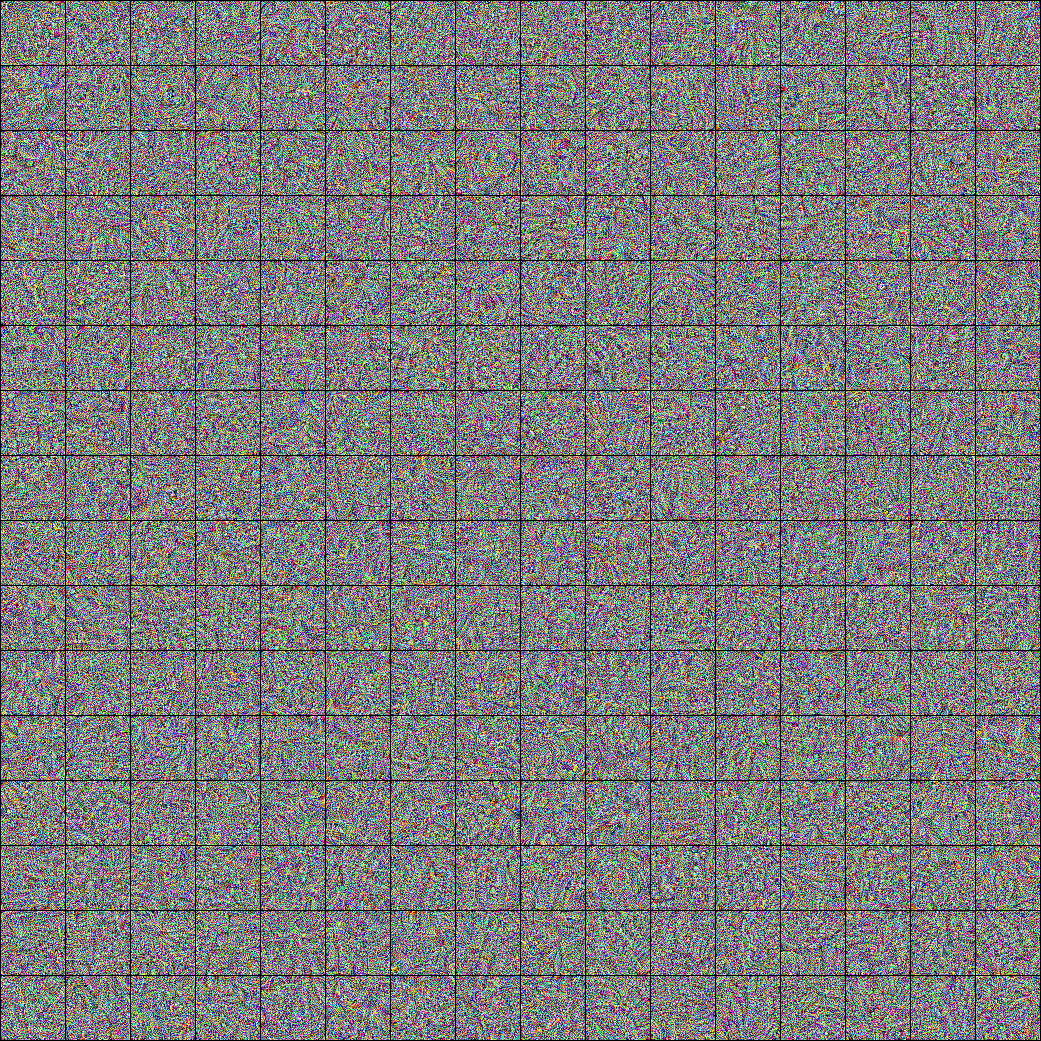}
\caption{\textbf{Pixel-space matched-moment optimization reverses the FID
ranking.} \refrev{The montages show 256 images from each 512-image set; metrics use all 512. Left: held-out
ImageNet images. Right: noise-initialized images optimized only to match the reference Inception mean and covariance.}
Despite being visual noise, the optimized set receives a lower (nominally better) FID ($24.7$ vs.
$58.6$). ZID gives it a larger departure score ($17.2$ vs. $0.86$), and
none of 499 random label permutations produces an equally large score; the held-out comparison has
permutation $p{=}.328$. See Sec.~\ref{sec:pixelgaming}.}
\label{fig:pixelgaming}
\end{figure}

A second limitation appears when the departure has a direction.
We sweep \citrev{classifier-free guidance (CFG) \citep{ho2022classifierfree}} in
\citrev{DiT-XL/2 \citep{peebles2023dit}}, a standard control of the
fidelity--diversity tradeoff. Along this sweep, FID and KID are scalar discrepancies that are unchanged when the two samples are exchanged and
therefore do not distinguish low-guidance off-manifold dispersion from high-guidance mode collapse.
ZID additionally reports a separate signed dispersion readout, which identifies the transition from over- to
under-dispersion (Fig.~\ref{fig:dit-cfg}).

Finite-sample sensitivity provides an additional concern: at $d=100$ and
$n=100$, KID power remains below $.25$ across pure-dispersion shrinkages from
$5\%$ to $25\%$, while FID reaches only $.75$ power for a $5\%$ shrinkage even at $n=1000$
(App.~\ref{app:power}).

ZID (\emph{Z-resolved Integrated Diagnostic}) addresses these issues by
combining complementary graph- and kernel-based statistics, calibrating their aggregated response by
permutation, and retaining a signed dispersion contrast. The aggregated response provides an index
of departure magnitude, and its permutation distribution yields an equality-test $p$-value. The
signed contrast permits an under- or over-dispersion label when the direction is identifiable.
Section~\ref{sec:method} defines the statistics and their
aggregation.

\paragraph{Contributions.}
We argue that generative-model evaluation should support model ranking and
calibrated testing, remain sensitive to distinct forms of distributional departure, and provide directional
diagnostics when the relevant effect is identifiable.
\begin{enumerate}\itemsep2pt
\item \textbf{Matched-moment analysis and calibrated detection.}
\refrev{FID's manipulability is known}
\citrev{\citep{alfarra2022robustness,kynkaanniemi2023role}},
\refrev{and prior work has exhibited distinct distributions with matching first two moments and hence zero FID}
\citrev{\citep{benny2021conditional,luzi2023gaussianmixtures,jayasumana2024cmmd}}.
\refrev{Proposition~\ref{prop:matched-moment} states this non-identifiability for any reference distribution
with nonzero covariance}, and degree-$3$ KID is likewise zero for distributions matched through third order.
We introduce ZID as a calibrated detector for these and other distributional
departures. In a pixel-space stress test optimized to match the reference
Inception mean and covariance, FID falls below the real--real baseline, while ZID assigns the resulting
images a much larger score than held-out real images.
\item \textbf{Detection coverage and directional diagnosis.}
ZID combines rank-based, median-bandwidth, and fine-scale kernel evidence to
cover departures that favor different sensitivities. Across controlled evaluations spanning location,
dispersion, dependence, higher-order structure, multimodality, and off-manifold support, it provides the
broadest overall detection coverage.
Alongside the calibrated equality-test $p$-value, ZID reports separate location-sensitive and
dispersion-sensitive component magnitudes and, when direction is identifiable, labels the departure as under- or
over-dispersed.
\item \textbf{Ranking by severity within departure types.}
Across controlled sweeps of increasing severity, the ZID score provides
consistent ordering across all examined departure types, including settings where FID is flat or reversed.
On DiT-XL/2 and \citrev{SiT-XL/2 \citep{ma2024sit}}, it also tracks departure across the practically relevant guidance range, while
the signed dispersion readout identifies the associated collapse direction.
\end{enumerate}

ZID is therefore intended as a broad-coverage diagnostic when the departure type
is not known in advance and ranking, calibrated testing, and, when identifiable, dispersion-direction
diagnosis are all required.
Our experiments use both Inception and DINO embeddings and include five
pretrained generator families: \citrev{BigGAN \citep{brock2019large}},
\citrev{StyleGAN2-ADA (StyleGAN2: \citealp{karras2020analyzing}; ADA: \citealp{karras2020training})},
\citrev{a CIFAR DDPM \citep{ho2020denoising}}, DiT-XL/2, and SiT-XL/2.

The paper proceeds as follows. Section~\ref{sec:background} reviews related
metrics. Section~\ref{sec:method} defines the ZID score, its outer-permutation equality test, and the
signed readouts. Section~\ref{sec:dominance} evaluates detection, ranking, construction choices, and
robustness under controlled feature-space departures, with comparisons across datasets and feature
representations. Section~\ref{sec:real} studies pretrained generators, and Sec.~\ref{sec:pixelgaming} presents
the pixel-space FID stress test. The appendices provide proofs, experimental protocols, expanded results,
and per-class values.

\section{Background and related metrics}\label{sec:background}
FID compares fitted feature-space Gaussians through their means and
covariances \citep{heusel2017fid}, whereas KID is the MMD associated with a degree-$3$ polynomial
kernel \citep{binkowski2018kid}. Gaussian-kernel MMD and CMMD are kernel two-sample discrepancies
that differ in their kernels or feature representations
\citep{gretton2012kernel,jayasumana2024cmmd}.
\citrev{FID's finite-sample bias motivated the extrapolation-based $\mathrm{FID}_\infty$
correction \citep{chong2020effectively}. Robustness and representation concerns motivated R-FID
\citep{alfarra2022robustness} and related analyses \citep{kynkaanniemi2023role}, while compound FID
aggregates Fr\'echet distances across multiple feature levels \citep{nunn2021compound}. These alternatives
address estimation or representation concerns but retain unsigned Gaussian moment summaries.}
MIND instead uses sliced Wasserstein distances to improve
sample efficiency and robustness to moment matching \citep{berthet2026mind}, while ECS uses
characteristic-function transforms to probe moments and tails \citep{tam2026ecs}.
C2ST provides a learned comparison by testing whether a held-out
classifier can distinguish the two samples \citep{lopezpaz2017revisiting}.

Some metrics instead focus on particular aspects of diversity, coverage,
or dispersion.
\refrev{Vendi Score} summarizes sample diversity \citep{friedman2023vendi};
Precision/Recall separates fidelity from coverage
\citep{sajjadi2018assessing,kynkaanniemi2019improved}; and Density/Coverage refines their
neighborhood-based estimates \citep{naeem2020reliable}.
PERMDISP tests homogeneity of multivariate dispersion through distances
to group centroids \citep{anderson2006dispersion}.

\section{The ZID diagnostic: score, test, and signed readout}\label{sec:method}
\subsection{Candidate families and a common two-coordinate structure}\label{sec:two-coordinate}
We begin with three candidate tests for the ZID construction: the generalized
permutation-based kernel test (GPK) \citep{song2024generalized}, the
rank-in-similarity-graph edge-count test (RISE) \citep{zhou2023new}, and the generalized
edge-count test (GET) \citep{chen2017generalized}.
Although they use different pairwise representations, all three produce
the same form of two-coordinate summary of within-sample similarity.
Let~$X=\{x_i\}_{i=1}^m$ denote the reference sample,
$Y=\{y_i\}_{i=1}^n$ the generated sample, and $N=m+n$ the~pooled sample size.
For any of these tests, let $w_{ij}$ denote its pairwise weight on
the pooled observations, and define
\[
U_x=\sum_{i<j}w_{ij}\mathbf 1\{i,j\in X\},\qquad
U_y=\sum_{i<j}w_{ij}\mathbf 1\{i,j\in Y\}.
\]
GPK uses dense Gaussian-kernel weights, GET a sparse binary similarity graph, and RISE a
rank-weighted similarity graph. Thus each test yields $U=(U_x,U_y)^\top$ while encoding similarity
differently. Let $\mu$ and $\Sigma$ be the mean and covariance of $U$
under the label-permutation null; throughout, the subscript $0$ denotes expectation, variance, or
covariance under this null. Each test measures the joint displacement of the two
within-sample quantities with the Mahalanobis statistic
\[
T=(U-\mu)^\top\Sigma^{-1}(U-\mu).
\]
This two-dimensional statistic admits an established decomposition. Define
\[
 W=\frac{U_x}{m-1}+\frac{U_y}{n-1},\qquad
 D=U_x-U_y.
\]
For GET, $W$ is proportional to the weighted edge-count statistic; for RISE, it is the
corresponding weighted combination of the two within-sample rank totals. For GPK, the centered
coordinate $W-\mathbb E_0W$ is proportional to the unbiased empirical squared MMD
(Sec.~\ref{sec:signed-imbalance}). Define the standardized forms
\[
 Z_W=\frac{W-\mathbb E_0W}{\sqrt{\mathrm{Var}_0(W)}},
 \qquad
 Z_D=\frac{D-\mathbb E_0D}{\sqrt{\mathrm{Var}_0(D)}}.
\]
The two standardized coordinates are uncorrelated under the permutation null and satisfy
\[
T=Z_W^2+Z_D^2,\qquad \mathrm{Cov}_0(Z_W,Z_D)=0.
\]

Because both weights in $W$ are positive, $Z_W$ records aligned displacement of the two
within-sample quantities, whereas $Z_D$ records their opposing displacement. Interchanging $X$
and $Y$ leaves $T$ unchanged and reverses the sign of $Z_D$.
Pure location shifts tend to displace $U_x$ and $U_y$ in the same direction relative to their permutation-null expectations
and therefore respond primarily through $Z_W$. Scale-dominated alternatives can instead produce signal primarily in the
signed $Z_D$ coordinate, particularly as dimension grows. At low dimension, however, scale
alternatives can also produce appreciable aligned displacement in $U_x$ and $U_y$, yielding signal in
$Z_W$. Departures combining location and scale
changes may produce signal in one or both coordinates.

\subsection{Connections and contrasts with MMD, FID, and KID}\label{sec:signed-imbalance}
Set $w_{ij}=k(z_i,z_j)$ and define
$U_{xy}=\sum_{i\in X,j\in Y}w_{ij}$. Let
$\bar k=\{N(N-1)\}^{-1}\sum_{i\ne j}k(z_i,z_j)$ be the fixed pooled average kernel weight.
Because $\mathbb E_0(W)=N\bar k/2$, the usual unbiased empirical squared MMD satisfies
\[
\widehat{\mathrm{MMD}}_u^2
=\frac{2U_x}{m(m-1)}+\frac{2U_y}{n(n-1)}-\frac{2U_{xy}}{mn}
=\frac{2(N-1)}{mn}\bigl(W-\mathbb E_0W\bigr).
\]
Thus, $\widehat{\mathrm{MMD}}_u^2$ and $W-\mathbb E_0W$ differ only by
the positive factor $2(N-1)/(mn)$.
MMD is symmetric under
swapping $X$ and $Y$ and supplies no signed-difference coordinate corresponding to $Z_D$.
KID uses the same MMD statistic with a degree-three polynomial kernel.
$Z_D$ instead contrasts $U_x$ with $U_y$; its sign supplies the orientation
used by the under- versus over-dispersion readout (Figs.~\ref{fig:schematic} and~\ref{fig:signed}).
The opposite-direction scale pattern is the mechanism exploited by
generalized graph and kernel tests \citep{chen2017generalized,song2024generalized}.
Under the scale alternatives, increasing dimension can concentrate pairwise similarities and place the two within-sample
similarities on opposite sides of their pooled permutation
reference. When these centered shifts have opposite signs, their contributions
largely cancel in $W-\mathbb E_0W$ but add in $D-\mathbb E_0D$. The same pattern can occur when
location and scale both change: opposing centered shifts can contribute to $Z_D$, while
aligned shifts simultaneously produce signal in $Z_W$.
App.~\ref{app:power} shows empirically how the scale response becomes
increasingly concentrated in $Z_D$ as dimension grows.

\begin{figure}[!t]\centering
\includegraphics[width=0.96\textwidth]{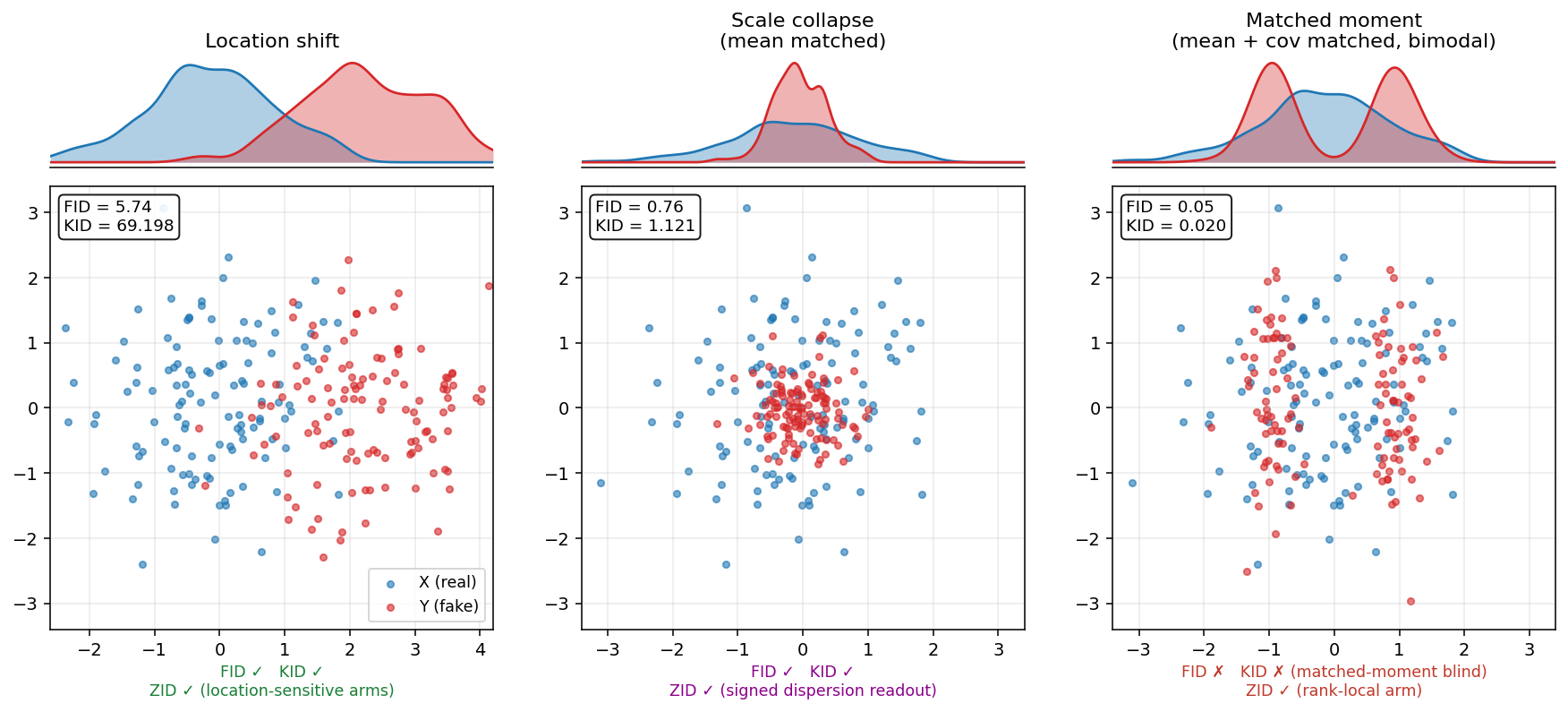}
\caption{\textbf{Controlled differences and diagnostic readouts.}
Two-dimensional schematics (real in blue, generated in red); FID and KID
are computed from the displayed samples. All three procedures respond to a location shift.
Under mean-matched scale collapse, the three scalar scores do not encode direction;
ZID's separate signed dispersion readout identifies the change as under-dispersion.
For the symmetric matched-moment construction, moments through total degree three also agree, so population FID and degree-$3$ KID are both zero.
ZID is not constrained to zero by either moment identity. High-dimensional finite-sample power for dispersion changes is quantified in
App.~\ref{app:power}.}
\label{fig:schematic}
\end{figure}

FID's dependence on only the feature mean and covariance creates a
matched-moment blind spot: its population value is exactly zero whenever those moments match, even if
the two feature distributions differ.
\refrev{The following elementary proposition records this known non-identifiability in a general form.}

\begin{proposition}[Blindness of two-moment discrepancies]
\label{prop:matched-moment}
Let $M(P,Q)$ be a discrepancy determined only by the means and covariances
of $P$ and $Q$, with $M(P,Q)=0$ whenever these moments agree. For every distribution $P$ on
$\mathbb R^d$ with finite second moments and nonzero covariance, there exists $Q\neq P$ with the same
mean and covariance. Hence $M(P,Q)=0$; in particular, $\operatorname{FID}(P,Q)=0$.
\end{proposition}

\noindent(Proof in Appendix~\ref{app:proof}: choose either the Gaussian with the same moments or,
when $P$ is already that Gaussian, an explicit non-Gaussian distribution with the same moments.)

KID uses a degree-$3$ polynomial kernel, so Prop.~\ref{prop:matched-moment} does not cover it directly.
If two distributions match all mixed moments of total degree at most three,
their population polynomial-kernel MMD, and hence population KID, is zero.

To examine finite-sample behavior under first-two-moment matching in real feature space,
we use a bimodal alternative whose mean and covariance match those of real CIFAR--Inception features. As
$m=n$ increases from $2$k to $25$k, empirical FID remains at its real--real baseline: the alternative/real--real FID values are
$24.5/24.3$ at $n=2$k and $2.00/1.99$ at $n=25$k. At $n=2$k, the ZID score defined below has median $84.3$
for the alternative and $.32$ for the corresponding real--real comparison; its permutation test
rejects in all 20 repetitions at the minimum attainable $p=.01$ with 99 permutations
(App.~\ref{app:feature-gaming}).

\subsection{Score construction and design choices}\label{sec:score-design}
ZID retains three complementary similarity representations: rank-based
RISE, GPK-med, and GPK-small.
RISE and GET both use similarity graphs. ZID retains RISE because its rank weights preserve graded
similarity information, whereas GET uses binary edge weights;
\citet{zhou2023new} report higher power for this weighted construction in most of their studied settings.
GET is not retained in the six-arm construction and is evaluated separately in Sec.~\ref{sec:dominance}.
GPK-med uses a Gaussian-kernel bandwidth equal to the pooled median
pairwise distance, denoted $\sigma_{\rm med}$. GPK-small uses
$0.175\sigma_{\rm med}$ in the reported six-arm ZID construction.
Sensitivity to nearby bandwidths is reported in App.~\ref{app:frozen-design-controls}.
For each representation, ZID retains both $Z_W^{(j)}$ and $Z_D^{(j)}$, giving six arms.

\paragraph{The ZID score: an omnibus departure index.}
Each of the six arms
$\{Z_W^{(j)},Z_D^{(j)}\}_{j=1}^3$ is centered and scaled to have a unit-variance permutation null; this
standardization is not, by itself, a normality claim. For arm $k$, define the two-sided
normal-reference tail $r_k=2\bar\Phi(|Z_k|)$, sort these as
$r_{(1)}\le\cdots\le r_{(6)}$, and set
\[
q_{\mathrm{rank}}=\min\!\left\{1,\min_{i=1,\ldots,6}\frac{6r_{(i)}}{i}\right\},
\qquad S_{\mathrm{ZID}}=-\log q_{\mathrm{rank}} .
\]
The flat-Simes form \citep{simes1986improved} makes the score coordinatewise monotone in the observed arm magnitudes and
unbounded. Larger values indicate a more extreme standardized departure
from the reference under the fixed evaluation protocol. Because the arms are
dependent and the normal-reference tails are reference mappings rather than finite-sample arm $p$-values,
$r_k$ and $q_{\mathrm{rank}}$ serve as reference-scale ranking quantities; the outer permutation in
Sec.~\ref{sec:permutation-readouts} supplies test calibration. Score comparisons use a common sample
size, embedding, reference, and preprocessing because an alternative's standardized departure can
vary with $n$.

Because different departures can concentrate signal in different arms, we use flat-Simes to preserve a strong
signal in any one arm without allowing another arm to cancel it.
The map $-\log\{2\bar\Phi(|Z|)\}$ puts unit-variance arms on a common, unbounded reference scale,
assigns zero severity at $Z=0$, and satisfies
$-\log\{2\bar\Phi(|Z|)\}=Z^2/2+\log|Z|+O(1)$ as $|Z|\to\infty$. Flat-Simes is
symmetric in the arms, coordinatewise monotone, and sparse-adaptive: one large arm can drive the score,
while several moderate arms reduce its multiplicity penalty. In particular, writing
$s_k=-\log\{2\bar\Phi(|Z_k|)\}$ and $s_{\max}=\max_k s_k$,
\begin{equation}
s_{\max}-\log 6\;\le\;S_{\mathrm{ZID}}\;\le\;s_{\max},
\end{equation}
and if all six severities are equal to $s$, then $S_{\mathrm{ZID}}=s$. Thus no arm can be hidden by
cancellation. Appendix~\ref{app:frozen-design-controls} compares flat Simes
with max and sum-of-squares aggregation, \refrev{Fisher's combination rule
\citep{fisher1932statistical}}, and the \refrev{Cauchy combination rule
\citep{liu2020cauchy}}. Flat Simes is among
the leading rules for both detection and ranking across increasing severity.

\subsection{Permutation test and diagnostic readouts}\label{sec:permutation-readouts}
For a pooled sample, we recompute all six arms and $S_{\mathrm{ZID}}$ after each of $B$ label
permutations.
The reported Monte Carlo permutation tail is denoted
$\widehat p_{\mathrm{ZID}}$:
\[
\widehat p_{\mathrm{ZID}}=\frac{1+\sum_{b=1}^{B}\mathbf 1\{S_{\mathrm{ZID}}^{(b)}
\ge S_{\mathrm{ZID}}^{\mathrm{obs}}\}}{B+1}.
\]
At nominal level $\alpha$, the test rejects when
$\widehat p_{\mathrm{ZID}}\le\alpha$. The \refrev{plus-one Monte Carlo permutation $p$-value
\citep{phipson2010permutation}} includes the
observed labeling alongside the $B$ random relabelings, giving $B+1$ exchangeable configurations.
Recomputing the entire construction under each relabeling preserves
its internal dependence and yields a finite-sample-valid Monte Carlo permutation $p$-value without
independence or Gaussian assumptions on the arms.
The test asks whether the embedded distributions are distinguishable at
the available $n$. ``ZID test'' and ``ZID score'' always denote the six-arm implementation;
member-specific analyses are labeled RISE or GPK.

\paragraph{Component readouts.}
Applying flat-Simes separately to
$\{Z_W^{(j)}\}_{j=1}^3$ and $\{Z_D^{(j)}\}_{j=1}^3$ gives nonnegative aggregated $W$- and $D$-component magnitudes,
denoted $S_W$ and $S_D$, respectively. $S_W$ summarizes the three $W$-coordinate responses; $S_D$ summarizes the three $D$-coordinate responses and supplies the magnitude used by the dispersion diagnosis below. \refrev{Let $\widehat p_W$ and $\widehat p_D$ denote the permutation tails of $S_W$ and $S_D$, respectively.} These component tails describe where the six-arm signal appears; only $\widehat p_{\mathrm{ZID}}$, the permutation tail of the complete six-arm score, determines whether the test rejects distributional equality, and neither magnitude is an additive part of that score.
The dispersion diagnosis applies the coherence rule below to $S_D$ and the signed member-level $D$ coordinates. In reported numerical
results, $S_W$ and $S_D$ denote the aggregated component magnitudes, whereas
$Z_W^{(j)}$ and $Z_D^{(j)}$ retain their signed member-level meanings; RISE,
GPK-med, and GPK-small superscripts identify the member arms.
For continuous visualization, the ungated net signed dispersion score is
$\operatorname{sign}(\sum_j Z_D^{(j)})S_D$. The formal diagnosis below combines the
$D$-component permutation tail with the reference-active member signs. The resulting diagnostic
readout is denoted
\[
D_{\mathrm{ZID}}=
\begin{cases}
-S_D, & \text{under-dispersion},\\
+S_D, & \text{over-dispersion},
\end{cases}
\]
and is undefined when the diagnosis is \emph{member-sign conflict}, \emph{ambiguous}, or
\emph{not assigned}. Thus $Z_D^{(j)}$ always denotes a standardized member coordinate, whereas
$D_{\mathrm{ZID}}$ denotes the coherence-qualified aggregate diagnostic.

\paragraph{Dispersion diagnosis.}
\refrev{Using the $D$-component tail $\widehat p_D$ of the flat-Simes magnitude formed
from the three $D$ arms $\{Z_D^{(j)}\}_{j=1}^3$,} define the descriptive reference-active set
$\mathcal A=\{j:2\bar\Phi(|Z_D^{(j)}|)\le \alpha\}$. Then
\[
\operatorname{diagnosis}=
\begin{cases}
\textit{under-dispersion}, & \widehat p_D\le \alpha,\ \mathcal A\ne\varnothing,\ Z_D^{(j)}<0\ \forall j\in\mathcal A,\\
\textit{over-dispersion}, & \widehat p_D\le \alpha,\ \mathcal A\ne\varnothing,\ Z_D^{(j)}>0\ \forall j\in\mathcal A,\\
\textit{member-sign conflict}, &
\widehat p_D\le \alpha\ \text{and signs within }\mathcal A\text{ conflict},\\
\textit{ambiguous}, &
\widehat p_D\le \alpha\ \text{and }\mathcal A=\varnothing,\\
\textit{not assigned}, & \widehat p_D>\alpha.
\end{cases}
\]
The six-arm $\widehat p_{\mathrm{ZID}}$ alone determines rejection of distributional equality;
$\widehat p_D$ is a prespecified diagnostic gate rather than an additional separately error-controlled test. An
under- or over-dispersion label, and hence $D_{\mathrm{ZID}}$, is reported only when $\widehat p_D\le\alpha$
and the reference-active member signs are coherent.
The same $\alpha$ threshold defining $\mathcal A$ filters signs by their normal-reference magnitude; it is not a
separate significance test. The labels \emph{under-dispersion} and
\emph{over-dispersion} describe the sign of the within-sample-similarity imbalance. In the controlled
scale, truncation, and guidance families studied here, this sign tracks diversity contraction or
expansion. Controlled scale changes, DiT, and SiT support this
interpretation. When reference-active signs conflict, the diagnosis is
\emph{member-sign conflict}, and the member signs are reported rather than collapsed into a single
direction. The label \emph{ambiguous} is reserved for a passed diagnostic gate with no
reference-active member from which to read a direction.

\section{Feature-space evaluation under controlled departures}\label{sec:dominance}
Using controlled transformations of real image features, we compare ZID
with its comparators across eight distributional departures and examine the directional diagnosis
supplied by ZID. We then evaluate ranking across sweeps of increasing severity, analyze the ZID
construction, vary severity, dimension, and sample size, and compare responses across datasets and
feature representations.

\subsection{Detection and directional diagnosis}
Each repetition draws two disjoint samples from a fixed pool of Inception
features extracted from real CIFAR-10 images \citep{krizhevsky2009learning} and reduced to $d=128$; a controlled transformation is
applied to the second sample. The common setting uses $m=n=200$, 300 independent repetitions, and
$\alpha=.05$. Permutation-calibrated methods use 499 label relabelings, and rejection follows the
$p\le\alpha$ convention defined in Sec.~\ref{sec:permutation-readouts}. We examine eight departure
families spanning location, dispersion, linear and nonlinear dependence, skewness, kurtosis,
matched-moment multimodality, and off-manifold support. Each transformation emphasizes one
distributional characteristic, whereas a real generator may combine several.
Preliminary multi-method pilot runs selected the displayed strength for each
family so that the maximum estimated power over the panel was high but below ceiling; the chosen signal
was then held fixed across all methods in the reported comparison.
The resulting signal strengths are specific to their departure families and
are not numerically comparable across columns.
\refrev{Appendix~\ref{app:controlled-departures} summarizes the controlled transformations;}
Sec.~\ref{sec:robustness} varies severity, dimension, and sample size.

\begin{figure}[!t]\centering
\includegraphics[width=\linewidth]{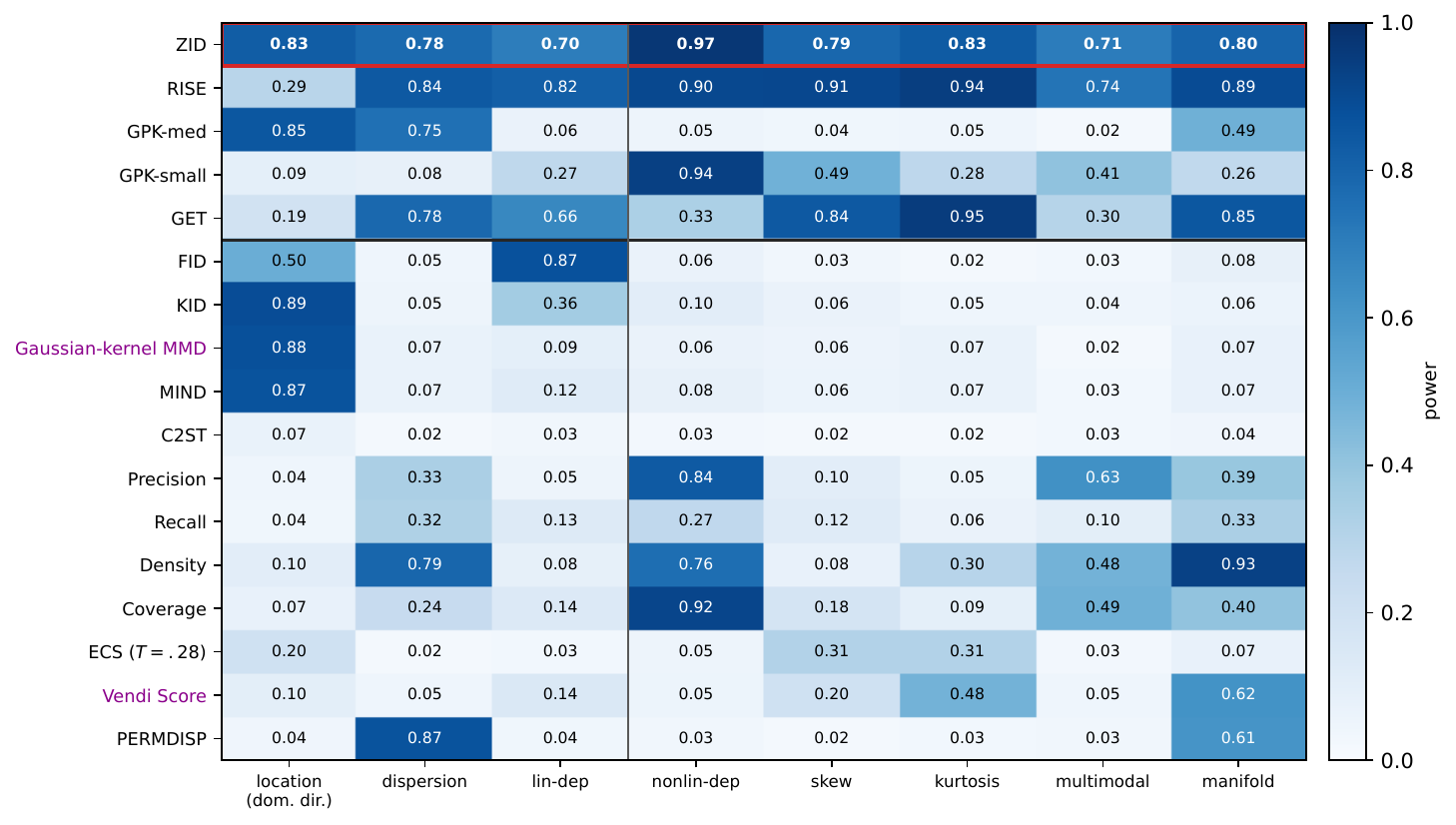}
\caption{\textbf{Detection power across eight controlled distributional departures.}
Cells report empirical power at the $0.05$ significance level using real CIFAR-Inception features reduced to
$d=128$, with $n=200$ observations per sample and 300 independent repetitions.
Permutation-calibrated tests use 499 relabelings.
The boxed row
is the complete six-arm ZID test. The next four rows report standalone RISE, GPK-med, GPK-small,
and GET tests; the remaining \refrev{twelve} are external comparators. The eight probes span location, dispersion, dependence, higher-order structure, multimodality,
and off-manifold support.
}
\label{fig:dom}
\end{figure}

Fig.~\ref{fig:dom} compares ZID with standalone RISE, GPK-med, GPK-small, and GET tests
and \refrev{twelve} external comparators. The general-purpose group
comprises FID, KID, \refrev{Gaussian-kernel MMD}, MIND, ECS, and C2ST; the specialist diagnostics are
Precision, Recall, Density, Coverage, \refrev{Vendi Score}, and PERMDISP
\citep{kynkaanniemi2019improved,naeem2020reliable,friedman2023vendi}.
Collectively, this panel draws on moments, kernels, projections,
classifiers, nearest-neighbor geometry, support estimates, diversity summaries, and dispersion tests.
ECS uses frequency parameter $T$, selected in the independent balancing study
documented in App.~\ref{app:ecs-sensitivity}. These four standalone tests use
their established calibrations: RISE and GET compare $Z_W^2+Z_D^2$ with their asymptotic
$\chi^2_2$ references, whereas GPK-med and GPK-small calibrate their quadratic statistics by label
permutation. The remaining external scalar and geometric comparators use sample-label permutation
calibrations of their displayed statistics. For the full six-arm ZID, the complete score is recomputed under every relabeling and
calibrated by its own outer permutation tail. \refrev{C2ST uses a stratified training--test split and a
fixed one-hidden-layer MLP; its held-out accuracy is calibrated by balanced test-label relabelings.}
\refrev{Appendix~\ref{app:c2st-protocol} summarizes the controlled-departure and comparator protocols.}

Relative to the \refrev{twelve}
external comparators, ZID has the highest power
on nonlinear dependence ($.97$), skewness ($.79$), kurtosis ($.83$), and matched-moment
multimodality ($.71$). KID leads location ($.89$), PERMDISP dispersion ($.87$), FID linear
dependence ($.87$), and Density off-manifold support ($.93$). ZID is the only row in this comparison
with power of at least $.70$ in every column; among the external comparators, the largest row minimum
is Density's $.08$.

Among these standalone tests, RISE is sensitive to dependence, higher-order structure, multimodality, and off-manifold support
but is weak on location ($.29$); GPK-med supplies location and dispersion
($.85/.75$); and GPK-small is strongest on nonlinear dependence ($.94$)
but weak on location and
dispersion. GET is broadly sensitive but weak on location, nonlinear dependence, and multimodality.
The six-arm aggregation combines these departure-specific
strengths into the breadth summarized above. Section~\ref{sec:design-sensitivity} tests this
complementarity by removing both arms of RISE, GPK-med, or GPK-small in turn and by analyzing severity ordering.

\paragraph{Null calibration on real features.}\label{sec:calibration}
To verify calibration at the representation and sample size used in
Fig.~\ref{fig:dom}, we compare two disjoint samples from the same real CIFAR--Inception PCA-128 pool
($m=n=200$) over 500 repetitions with 999 outer permutations, giving finer Monte Carlo
$p$-value resolution at this setting. The resulting rejection rates are
$.048$ at $\alpha=.05$ (Wilson 95\% CI $[.032,.070]$) and $.010$ at $\alpha=.01$.

\paragraph{Directional diagnosis.}
Detection power does not assess whether a metric identifies the direction
of a scale change. We therefore apply symmetric contraction and expansion to full-dimensional
Inception features and compare the resulting magnitudes and signed readout (Fig.~\ref{fig:signed}).
In this full-rank dispersion sweep, the ungated net signed dispersion score changes sign across
the reference scale. The formal readout $D_{\mathrm{ZID}}$ is negative for the displayed contractions
and positive for the displayed expansions; at $c{=}1$, it is undefined because no direction is assigned. FID
remains unsigned and can assign similar magnitudes to deviations on opposite sides of the
reference.

\begin{figure}[!t]\centering
\includegraphics[width=0.50\linewidth]{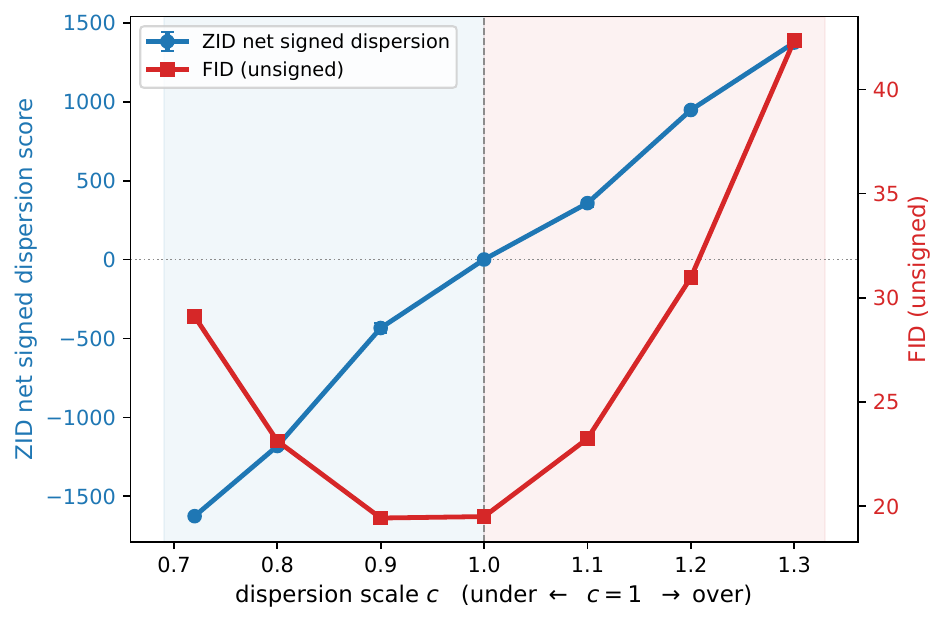}
\caption{\textbf{Signed vs.\ unsigned dispersion}
(real raw Inception-$2048$, disjoint $m{=}n{=}2500{>}d$, full-rank covariance;
$Y{=}\mu{+}c(\mathrm{base}{-}\mu)$). The blue curve shows the ungated net signed dispersion score,
$\operatorname{sign}(\sum_j Z_D^{(j)})S_D$. The formal readout $D_{\mathrm{ZID}}$ equals this score when the diagnosis resolves
to under- or over-dispersion; at $c=1$, the diagnosis is not assigned and $D_{\mathrm{ZID}}$ is undefined
(Sec.~\ref{sec:permutation-readouts}).
FID is unsigned and is nearly equal at $c{=}0.8$ and $1.1$.
FID's displayed minimum, $19.43$ at $c{=}0.9$ versus $19.50$ at $c{=}1$, slightly favors mild
collapse in this finite sample. Points are means over eight independent repetitions.}
\label{fig:signed}
\end{figure}

\subsection{Ranking across increasing severity}\label{sec:ranking}
Detection and ranking address different questions.
Figure~\ref{fig:dom} compares rejection power at one level of each departure. The ranking experiments use
CIFAR--Inception PCA-128 features with $m=n=200$. For each controlled departure family, the full sweep
evaluates six ordered levels, beginning with equality and increasing in severity, with 100 independently generated
reference--comparison pairs per level; Spearman $\rho$ is computed between the score and the
level index across the resulting 600 comparisons (Fig.~\ref{fig:rhobars}). A separate endpoint experiment
generates 100 independent mild-versus-final comparison pairs per family and reports the fraction for
which the final-level score is at least as large as the mild-level score (Table~\ref{tab:ranking}).
The transformations and complete sweeps are specified in App.~\ref{app:controlled-departures}.

\begin{figure}[!t]\centering
\includegraphics[width=\linewidth]{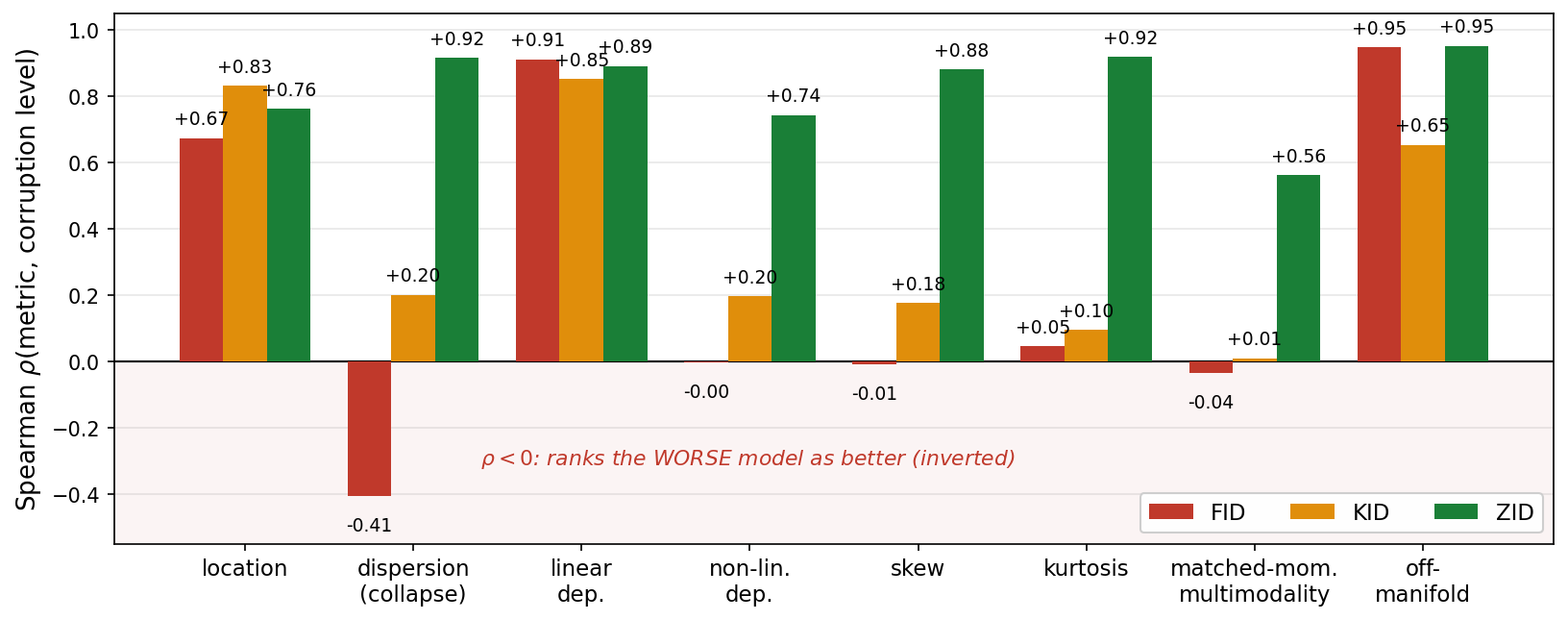}
\caption{\textbf{Ranking across sweeps of increasing severity.} Spearman $\rho$
between each score and the six ordered levels of each controlled departure family. Positive values
indicate that the score tends to increase as the departure becomes more severe; negative values indicate
that it tends to assign smaller scores to more severe samples. Within each departure, all methods use the
same embedding, sample size, and preprocessing.}
\label{fig:rhobars}
\end{figure}

Figure~\ref{fig:rhobars} shows that FID and KID track increasing severity on
location, linear dependence, and off-manifold support. FID instead reverses the dispersion-collapse
ordering ($\rho=-.41$) and is
essentially flat ($|\rho|\le .05$) on the four moment-restored departures: nonlinear dependence,
skewness, kurtosis, and matched-moment multimodality. KID remains weak on dispersion collapse and the
four moment-restored sweeps ($\rho\le .20$). The ZID score maintains positive rank association on every
sweep, ranging from $.56$ on
matched-moment multimodality to $.95$ on off-manifold support.

\begin{table}[!t]\centering\footnotesize
\caption{\textbf{Pairwise ordering of graded departures.}
Accuracy is the percentage of 100 independent mild-versus-final comparisons
ordered correctly; chance accuracy is $50\%$.
The real CIFAR-Inception protocol is $d{=}128,n{=}200$; dispersion is diversity
collapse. Bold denotes the row best.}
\label{tab:ranking}
\begin{tabular}{lccc}
\toprule
 & FID & KID & ZID \\
\midrule
location & 98 & \textbf{100} & 99 \\
dispersion (collapse) & 27 & 64 & \textbf{100} \\
linear dependence & \textbf{100} & \textbf{100} & \textbf{100} \\
nonlinear dependence & 45 & 61 & \textbf{92} \\
skewness & 50 & 61 & \textbf{100} \\
kurtosis & 52 & 48 & \textbf{100} \\
matched-mom.\ multimodality & 55 & 59 & \textbf{83} \\
off-manifold support & \textbf{100} & \textbf{100} & \textbf{100} \\
\bottomrule
\end{tabular}
\end{table}

Table~\ref{tab:ranking} shows that all three methods are near-perfect on
location, linear dependence, and off-manifold support. On dispersion collapse and the four
moment-restored departures, ZID orders $83$--$100\%$ of pairs correctly, whereas FID is at or below
chance ($27$--$55\%$) and KID reaches only $48$--$64\%$.

\begin{samepage}
To broaden the comparison beyond FID and KID, we apply both ranking protocols
to these four standalone tests and three additional external comparators: multiscale MMD, MIND,
and ECS (App.~\ref{app:frozen-design-controls}). Among the five external comparators, ZID has the highest
mean and minimum Spearman correlations across the eight departures, as well as the highest mean and minimum
pairwise ordering accuracies. Among these standalone tests, RISE is essentially tied with ZID in mean
pairwise ordering accuracy ($.969$ versus $.968$) and has higher minimum accuracy ($.90$ versus $.83$), whereas
ZID has higher mean and minimum Spearman correlations.
\end{samepage}

\subsection{Construction ablations and sensitivity}\label{sec:design-sensitivity}
To examine the roles of RISE, GPK-med, and GPK-small and the two coordinates
within ZID, Table~\ref{tab:arm-selection} compares the full six-arm construction with variants that omit
RISE, GPK-med, or GPK-small and two coordinate-only variants. Table~\ref{tab:arm-severity} isolates the contribution
of GPK-small by comparing the full six-arm ZID with its variant without GPK-small across four severity levels.
Table~\ref{tab:memorize} studies a moment-restored resampling-with-perturbation construction.
Starting from two disjoint real samples $R$ and $H$, with $R$ as the reference, the construction
replaces a fraction $f$ of the observations by perturbed
resamples from $H$ and then restores the sample mean and covariance of $H$
(App.~\ref{app:controlled-departures}). The table examines how the
full six-arm ZID, RISE, GPK-med, and GPK-small, together with FID and KID, order increasing values of $f$
across perturbation scales. Finally, we check sensitivity to the Gaussian-reference tail
probabilities by replacing them with permutation estimates (App.~\ref{app:controlled}).

\begingroup
\setcounter{bottomnumber}{2}
\renewcommand{\bottomfraction}{0.90}
\renewcommand{\textfraction}{0.05}
\begin{table}[!t]\centering\small
\caption{\textbf{Paired construction ablation at the Fig.~\ref{fig:dom}
signal levels.}
Results use real CIFAR-Inception features at $d=128,n=200,\alpha=.05$, with
300 repetitions and 499 outer permutations. Rejection uses $p\le.05$.
Within each repetition, every specification uses the same
samples, six computed arms, and label permutations.
Among the rows omitting RISE, \refrev{GPK-med}, or \refrev{GPK-small}, bold marks the largest power loss
from the full six-arm ZID for each departure.}
\label{tab:arm-selection}
{\renewcommand{\arraystretch}{0.90}\setlength{\tabcolsep}{2.5pt}%
\begin{tabular}{@{}lrrrrrrrrr@{}}
\toprule
specification & location & dispersion & lin.-dep. & nonlin.-dep. & skew & kurtosis &
multimodal & off-manifold & mean \\
\midrule
full six-arm ZID &
.830 & .750 &
.727 & .980 &
.847 & .847 &
.717 & .833 &
.816 \\
drop RISE &
.843 & \textbf{.630} &
\textbf{.167} & .893 &
\textbf{.347} & \textbf{.210} &
\textbf{.317} & \textbf{.527} &
.492 \\
drop \refrev{GPK-med} &
\textbf{.237} & .760 &
.800 & .983 &
.893 & .907 &
.767 & .870 &
.777 \\
drop \refrev{GPK-small} &
.840 & .800 &
.740 & \textbf{.860} &
.840 & .897 &
.697 & .853 &
.816 \\
\midrule
$Z_W$ arms only &
.850 & .073 &
.793 & .983 &
.830 & .747 &
.777 & .100 &
.644 \\
$Z_D$ arms only &
.063 & .837 &
.127 & .467 &
.407 & .787 &
.130 & .910 &
.466 \\
\bottomrule
\end{tabular}}
\end{table}

\begin{table}[!t]\centering\small
\caption{\textbf{Paired contribution of GPK-small across severity levels.}
Entries are paired rejection-rate differences,
$\Delta=\operatorname{Power}(\text{full six-arm ZID})-
\operatorname{Power}(\text{ZID without GPK-small})$. Positive values favor including \refrev{GPK-small}.
The first seven departures use perturbation multipliers $0.50,0.75,1.00,1.25$; matched-moment
multimodality replaces fractions $.25,.50,.75,1.00$ of the initially real observations with their
moment-matched bimodal counterparts. All cells use 300 paired repetitions,
499 outer permutations, and rejection at $p\le.05$. \textbf{Bold} marks $|\Delta|\ge.05$.}
\label{tab:arm-severity}
{\renewcommand{\arraystretch}{0.90}
\begin{tabular}{lrrrr}
\toprule
departure & setting 1 & setting 2 & setting 3 & setting 4 \\
\midrule
common multiplier & $0.50$ & $0.75$ & $1.00$ & $1.25$ \\
\cmidrule(lr){1-5}
location & $-.013$ & $-.027$ & $-.013$ & $+.003$ \\
dispersion & $-.040$ & $\mathbf{-.060}$ & $-.040$ & $-.017$ \\
linear dependence & $+.010$ & $-.023$ & $-.017$ & $.000$ \\
nonlinear dependence & $\mathbf{+.343}$ & $\mathbf{+.193}$ & $\mathbf{+.077}$ & $+.013$ \\
skewness & $.000$ & $-.027$ & $.000$ & $+.010$ \\
kurtosis & $-.020$ & $\mathbf{-.060}$ & $-.030$ & $.000$ \\
off-manifold support & $+.017$ & $-.020$ & $-.027$ & $.000$ \\
\midrule
multimodality replacement fraction & $.25$ & $.50$ & $.75$ & $1.00$ \\
\cmidrule(lr){1-5}
matched-moment multimodality & $.000$ & $+.023$ & $+.033$ & $+.030$ \\
\bottomrule
\end{tabular}}
\end{table}

\begin{table}[!htbp]\centering\small
\caption{\textbf{Moment-restored resampling with perturbation.} Spearman $\rho$ between each score and the replacement fraction $f$ (real CIFAR-Inception, $d{=}128$, $n{=}200$), evaluated across perturbation scales $\varepsilon$ (fractions of the per-coordinate standard deviations). Five perturbation scales $\varepsilon\in\{.05,.10,.20,.30,.50\}$ and six replacement fractions $f\in\{0,.1,.2,.3,.4,.5\}$ are evaluated with 100 repetitions at each $(\varepsilon,f)$ setting. Within each repetition, every method is evaluated on the same reference sample $R$ and transformed sample $B$. Their standalone columns use $T=Z_W^2+Z_D^2$. \textbf{Bold} marks the largest correlation in each row.}
\label{tab:memorize}
\begin{tabular}{cc|cc|ccc}
\toprule
$\varepsilon$ & \textbf{ZID} & FID & KID & RISE & \refrev{GPK-med} & \refrev{GPK-small} \\
\midrule
0.05 & $+.97$ & $-.10$ & $-.13$ & $+.74$ & $-.10$ & $\mathbf{+.98}$ \\
0.10 & $+.96$ & $-.04$ & $+.02$ & $+.71$ & $+.00$ & $\mathbf{+.97}$ \\
0.20 & $+.96$ & $+.04$ & $-.01$ & $+.66$ & $-.07$ & $\mathbf{+.97}$ \\
0.30 & $+.93$ & $+.02$ & $+.06$ & $+.55$ & $+.05$ & $\mathbf{+.95}$ \\
0.50 & $\mathbf{+.86}$ & $+.06$ & $+.04$ & $+.23$ & $-.07$ & $+.84$ \\
\bottomrule
\end{tabular}
\end{table}

Table~\ref{tab:arm-selection} shows complementary roles within the full six-arm
ZID. Omitting RISE produces the largest power loss on six of the
eight departures; omitting
\refrev{GPK-med} reduces location power from $.830$ to $.237$; and omitting \refrev{GPK-small} has its
clearest effect on nonlinear dependence, reducing power from $.980$ to $.860$. Among the coordinate-only
variants, the $Z_W$-only variant is stronger on location, dependence, skewness,
and multimodality, whereas the $Z_D$-only variant is stronger on dispersion, kurtosis, and off-manifold
support.
\endgroup

With bandwidth $0.175\sigma_{\rm med}$, \refrev{GPK-small} targets unusually
close-pair structure, complementing the rank-based and median-bandwidth members.
Table~\ref{tab:arm-severity} shows that \refrev{GPK-small} contributes most
clearly on nonlinear dependence: its incremental power contribution grows from $+.013$ at multiplier
$1.25$ to $+.343$ at $.50$, with the largest gain at the lower-signal end of the tested range.
Its effects on the other seven departures are smaller and occasionally negative.
Table~\ref{tab:memorize} provides a second view of the fine-scale role: under a
moment-restored resampling-with-perturbation construction, the full six-arm ZID score remains strongly
associated with the replacement fraction across perturbation scales ($\rho=.86$--$.97$). \refrev{GPK-small}
supplies the clearest member-level signal ($.84$--$.98$), while RISE weakens from $.74$ to $.23$ as the
perturbation scale increases; FID, KID, and \refrev{GPK-med} remain near zero.

\paragraph{Gaussian-reference versus permutation arm tails.}
For flat-Simes aggregation, ZID assigns each standardized arm the two-sided
Gaussian-reference tail probability $2\bar\Phi(|Z_k|)$. To check sensitivity to this choice, we compare
these probabilities with estimates obtained from inner label permutations. Both resulting scores use the
same outer calibration and therefore control Type-I error. They give broadly similar power profiles
(App.~\ref{app:controlled}). Because permutation estimation requires an additional inner loop, we use the
computationally simpler Gaussian-reference formula.

\FloatBarrier
\subsection{Sensitivity across severity and \texorpdfstring{$(d,n)$}{(d,n)} regimes}\label{sec:robustness}
We evaluate the full six-arm ZID over five ordered severity levels with
$d=128$, $n=200$, 300
repetitions, and 499 outer permutations per cell. Seven departures use multiples
$\{0,.5,.75,1,1.25\}$ of the Fig.~\ref{fig:dom} signal; matched-moment multimodality instead uses
fractions $\{0,.25,.5,.75,1\}$ of the initially real observations replaced by their
moment-matched bimodal counterparts; $\lambda=0$ is the unmodified real sample and $\lambda=1$ is the full
bimodal alternative, the highest-severity endpoint.
Across each five-level ladder, the empirical rejection rate increases monotonically with the perturbation
multiplier, or with the replacement fraction for matched-moment multimodality
(App.~\ref{app:severity}, Fig.~\ref{fig:severity-suite}).

To compare relative sensitivity across several dimension--sample-size regimes,
we repeat the eight-departure comparison at $d\in\{128,512,2048\}$ and
$n\in\{50,200,600\}$ (Fig.~\ref{fig:domrobust}). Signals are selected separately within each $(d,n)$ cell, so comparisons are
within cells rather than power trends over $d$ or $n$. Method definitions and calibration procedures
follow Fig.~\ref{fig:dom}. Rows requiring permutation calibration use 99 relabelings per repetition.
We define a cell-wise leading group among ZID and the \refrev{twelve} external
comparators using a threshold of 80\% of the highest power within that set. Across the 72 cells, ZID enters
this group 63 times, followed by Density (28), PERMDISP (17), and FID (11); each remaining external
comparator enters it at most nine times. When assessed against the same threshold, the four
standalone rows reach it in 58 cells for RISE, 50 for GET, 18 for GPK-med, and 13 for GPK-small.

\begin{table}[!b]\centering\small
\caption{\textbf{Cross-dataset comparison across four perturbations.}
Entries report empirical power at the $0.05$ significance level on CIFAR-10 and Tiny-ImageNet (200 classes), with dataset-specific
PCA-128 maps fitted to $5{,}000$ Inception features, $m=n=200$, 300 repetitions, and 499 relabelings.
For mode-drop, CIFAR-10 retains $80\%$ of its classes and Tiny-ImageNet retains $50\%$; the other three perturbations use the same displayed strengths in both datasets.
Bold marks row maxima.}
\label{tab:real}
\resizebox{\textwidth}{!}{%
\begin{tabular}{llcccccccc}
\toprule
dataset & alternative & ZID & FID & KID & MIND & Density & Recall & Coverage & C2ST \\
\midrule
CIFAR-10 & collapse ($s=.05$) &
\textbf{.97} & .08 &
.06 & .18 & .96 &
.68 & .18 &
.03 \\
 & over-disp. ($\tau=.25$) &
.68 & .07 &
.07 & .05 & \textbf{.86} &
.26 & .39 &
.05 \\
 & location ($\delta=.40$) &
.91 & .67 &
\textbf{.97} & .93 & .15 &
.05 & .10 &
.11 \\
 & mode-drop ($80\%$ retained; 2 classes dropped) &
.85 & \textbf{.93} &
.88 & .83 & .06 &
.10 & .86 &
.15 \\
\midrule
Tiny-IN & collapse ($s=.05$) &
.89 & .08 &
.05 & .18 & \textbf{.92} &
.62 & .15 &
.05 \\
 & over-disp. ($\tau=.25$) &
.62 & .07 &
.05 & .08 & \textbf{.80} &
.27 & .37 &
.05 \\
 & location ($\delta=.40$) &
\textbf{.94} & .54 &
.93 & .88 & .25 &
.09 & .05 &
.06 \\
 & mode-drop ($50\%$ retained; 100 classes dropped) &
.80 & \textbf{.94} &
.82 & .65 & .05 &
.28 & .33 &
.30 \\
\bottomrule
\end{tabular}}
\end{table}

\begin{figure}[!b]\centering
\includegraphics[width=\textwidth]{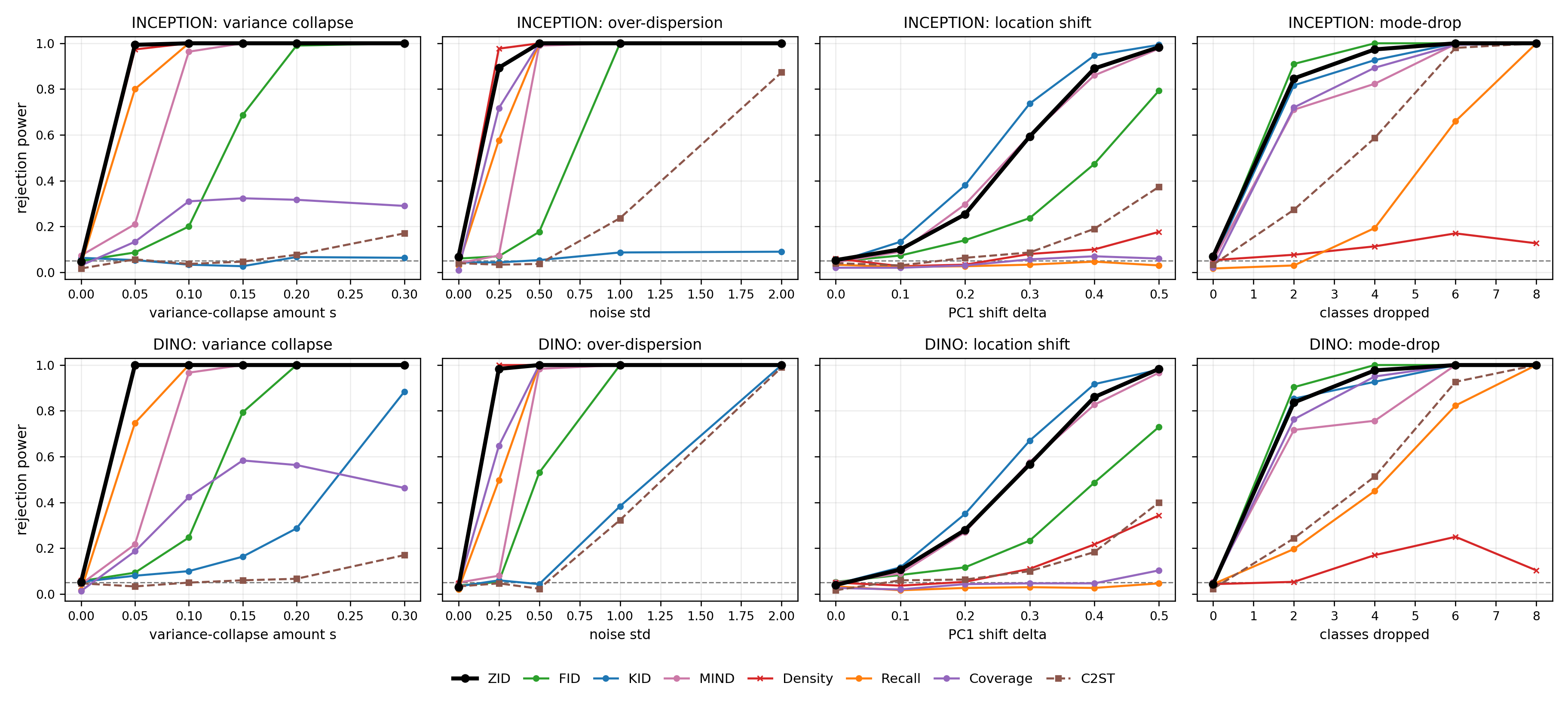}
\caption{\textbf{Detection power across perturbation levels on CIFAR-10.}
Rows use Inception-2048 and DINO ViT-S/16 CLS embeddings; columns show variance collapse
(variance-collapse amount $s$, where deviations from the reference-feature mean are multiplied by $1-s$), over-dispersion (adding independent
Gaussian noise with coordinate-wise standard deviation $\tau$ times the reference standard deviation),
location shift (translating by $\delta$ reference standard deviations along the leading PC), and
mode-drop (omitting the displayed number of classes). Curves report power at the $0.05$
significance level with $m=n=200$ and 300 repetitions per level. Permutation-calibrated methods use 99
relabelings.}
\label{fig:real}
\end{figure}

\subsection{Additional dataset and representation comparisons}\label{sec:real-embedding-transfer}

The controlled evaluations above use CIFAR-10 Inception features.
We next examine whether their response patterns persist under a different
dataset or feature representation. Table~\ref{tab:real} evaluates variance collapse,
over-dispersion, a leading-PC location shift, and class mode-drop on CIFAR-10 and Tiny-ImageNet
(200 classes) in Inception features\footnote{The Inception representation is identified in App.~\ref{app:controlled-departures}.
Absolute FID/KID magnitudes are on an implementation-specific
scale; the analysis uses their relative rankings, and Fig.~\ref{fig:real} reports the corresponding DINO evaluation.}. The two datasets show closely matched response profiles: ZID has high
power on collapse and location in both datasets, Density leads
over-dispersion, and FID leads mode-drop; ZID power remains $.62$--$.97$ across all eight comparisons.

Figure~\ref{fig:real} traces power across perturbation levels for the same four
departures on CIFAR-10 using Inception-2048 and DINO ViT-S/16 CLS embeddings
\citep{caron2021dino}. \refrev{The DINO representation is identified in App.~\ref{app:controlled-departures}.}
Across the full curves, ZID responds to all four departures in both representations:
power rises rapidly with variance collapse, over-dispersion, and class omission, and increases smoothly with
the leading-PC shift. At selected levels, Density leads over-dispersion, KID leads location shift, and FID leads
mode-drop.

\section{Evaluation of pretrained generative models}\label{sec:real}\label{sec:realgen}

We examine five evaluations of pretrained generative models spanning GAN, diffusion, and
interpolant-transformer architectures. They test truncation ordering, sensitivity to sampling steps, and
the transition from low-guidance over-dispersion to high-guidance collapse; Table~\ref{tab:external-validity}
summarizes each setting. Appendix~\ref{app:generator-protocol} gives the sampling, reference, and
calibration protocol.

\FloatBarrier
\begin{table}[H]\centering\footnotesize
\caption{\textbf{Evaluations of pretrained generative models.}
Each row varies only the listed sampling or truncation parameter; the pretrained
generator and evaluation setup are otherwise fixed.}
\label{tab:external-validity}
\setlength{\tabcolsep}{4pt}
\begin{tabular}{@{}>{\raggedright\arraybackslash}p{.13\linewidth}
>{\raggedright\arraybackslash}p{.09\linewidth}
>{\raggedright\arraybackslash}p{.08\linewidth}
>{\raggedright\arraybackslash}p{.11\linewidth}
>{\raggedright\arraybackslash}p{.16\linewidth}
>{\raggedright\arraybackslash}p{.17\linewidth}
>{\raggedright\arraybackslash}p{.15\linewidth}@{}}
\toprule
generator & paradigm & domain & varied parameter & embedding & evaluation scale & role \\
\midrule
BigGAN-deep-128 & GAN & ImageNet & truncation & Inception-2048 & 10 classes, $n=450$ & truncation ordering and direction \\
\cmidrule(lr){1-7}
CIFAR DDPM & diffusion & CIFAR-10 & DDIM steps & Inception-2048 + DINOv2-384 & $n=200$ & sampling-step comparison and direction \\
\cmidrule(lr){1-7}
DiT-XL/2 & diffusion transformer & ImageNet & CFG & Inception-2048 & 6 classes, $n=500$ & over- to under-dispersion \\
\cmidrule(lr){1-7}
SiT-XL/2 & interpolant transformer & ImageNet & CFG & Inception-2048 & 6 classes, $n=500$ & \refrev{guidance transition in a second architecture} \\
\cmidrule(lr){1-7}
StyleGAN2-ADA & GAN & FFHQ faces & truncation & Inception-2048 & $n=5000$ & full-rank evaluation on FFHQ \\
\bottomrule
\end{tabular}
\end{table}

\paragraph{BigGAN-deep-128: ordering along truncation.}
On \textbf{BigGAN-deep-128} \citep{brock2019large} ($m{=}n{=}450$, $10$ ImageNet classes),
FID is lowest at $\psi=.8$ and rises under stronger truncation. KID and the ZID score both increase
monotonically as truncation tightens, from $.0075$ to $.0173$ and from $328$ to $1150$, respectively.
The ZID test rejects all four generator--real comparisons at $p=.002$. ZID reports under-dispersion
throughout, with $S_D$ increasing from $42.2$ to $176.7$ (App.~\ref{app:biggan}).

\paragraph{CIFAR DDPM: sampling-step comparison and direction.}
On a pretrained CIFAR-10 \textbf{diffusion} model
\citep{ho2020denoising}, ZID rejects all three generator--real comparisons, and both the ZID score and
KID move toward their real--real values as DDIM steps increase \citep{song2021denoising}. FID likewise
decreases from $173.6$ to $130.8$ toward the real--real value $121.0$. ZID reports under-dispersion at
five steps and a member-sign conflict at 15 and 25 steps
(Table~\ref{tab:ddpm}).

\begin{table}[!b]\centering\small
\caption{\textbf{DDPM sampling-step sweep.}
\refrev{A pretrained CIFAR-10 DDPM versus real CIFAR-10 in Inception
$d{=}2048$, $m{=}n{=}200$). The ZID equality test and diagnostic gate each use 499 outer permutations.
The real--real row compares disjoint halves and has no assigned dispersion direction.}}
\label{tab:ddpm}
\begin{tabular}{lccccp{.25\linewidth}}
\toprule
DDIM steps & FID & KID & ZID score & \shortstack{ZID test\\$p$-value} & dispersion diagnosis \\
\midrule
5 & 173.6 & .0789 & 6442 & .002 & under-dispersion \\
\cmidrule(lr){1-6}
15 & 138.3 & .0223 & 519 & .002 & member-sign conflict \\
\cmidrule(lr){1-6}
25 & 130.8 & .0148 & 255 & .002 & member-sign conflict \\
\midrule
real--real & 121.0 & .0004 & .109 & .878 & not assigned \\
\bottomrule
\end{tabular}
\end{table}

\paragraph{DiT-XL/2: one guidance sweep, two failure directions.}
\label{sec:dit-cfg}

We vary classifier-free guidance (CFG) on \textbf{DiT-XL/2-256}
\citep{peebles2023dit} for six ImageNet classes ($n{=}500$), producing a trajectory from diffuse,
low-precision samples to high-guidance diversity collapse. At
$\mathrm{CFG}\in\{1,2,4,8\}$~\citep{ho2022classifierfree}, we compare each sample set with matched real images:
Higher Precision and Recall indicate better fidelity and diversity,
respectively, whereas smaller FID, KID, and ZID scores indicate less departure within this fixed sweep. Low guidance is over-dispersed/off-manifold,
while increasing guidance improves fidelity but collapses diversity.
In the class-aggregated curves in Fig.~\ref{fig:dit-cfg}(a), FID, KID, and
the ZID score attain their minima at $\mathrm{CFG}{=}2$. These scalar summaries identify the same
aggregate closest setting, but do not indicate whether larger departures reflect off-manifold
over-dispersion or diversity collapse. Precision and Recall must be read together: at
$\mathrm{CFG}{=}8$, Precision can remain as high as $0.89$ while Recall falls to $0.025$; at
$\mathrm{CFG}{=}1$, mean Recall is $0.75$ even though Precision is only $0.4$--$0.65$. Recall alone can
therefore make the low-guidance samples appear strong despite their off-manifold mass, just as Precision
alone can make the collapsed high-guidance samples appear strong.

\refrev{As CFG increases from $2$ to $8$,} ZID's
$W$-component magnitude $S_W$ rises through CFG 8 in all six classes. The final
aggregate dispersion readout is over-dispersion at CFG$=1$ and
under-dispersion at CFG$\in\{2,4,8\}$ for every class; the diagnostic readout $D_{\mathrm{ZID}}$ therefore changes from
positive at low guidance to negative as guidance rises.
Figure~\ref{fig:dit-cfg} is the counterpart of Fig.~\ref{fig:signed} for a pretrained generative model: departure
magnitude increases on both sides of the closest sampled setting, while the signed readout
distinguishes the two failure directions. Per-class results are in App.~\ref{app:dit-cfg}.

\begin{figure}[H]
\centering
\includegraphics[width=0.95\linewidth]{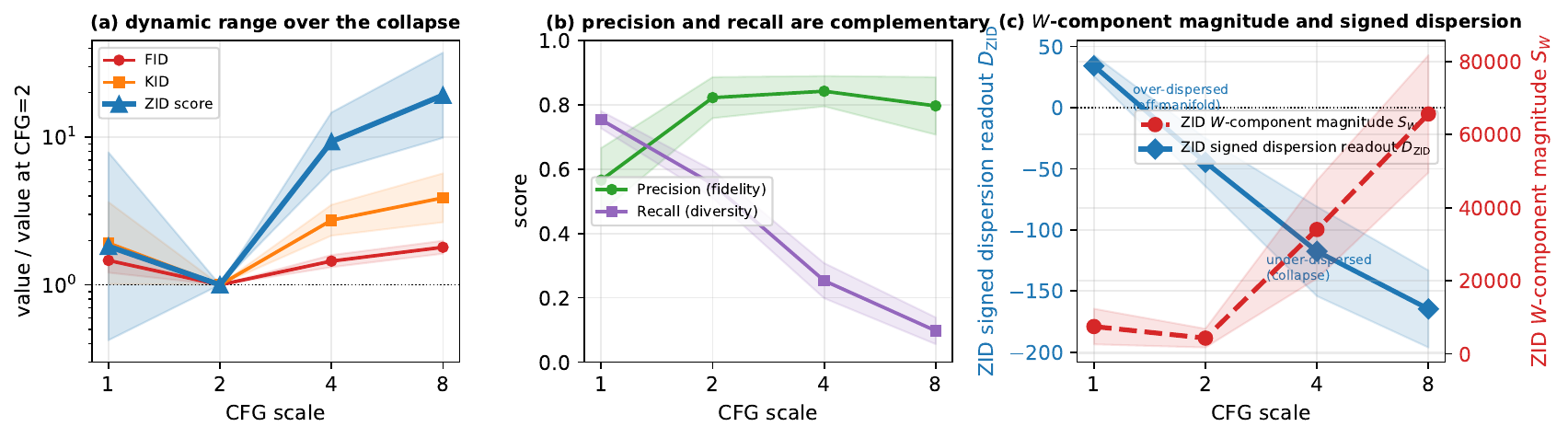}
\caption{\textbf{A guidance sweep from over- to under-dispersion.}
Across six DiT-XL/2 ImageNet classes and $\mathrm{CFG}\in\{1,2,4,8\}$:
(a) thick lines are geometric means across the six class-specific ratios to their CFG$=2$ values, and
shading spans one standard deviation on the log-ratio scale;
(b) Precision can remain high while Recall collapses; and
(c) ZID's nonnegative $W$-component magnitude $S_W$ tracks departure magnitude along this sweep, while its diagnostic readout $D_{\mathrm{ZID}}$
flips from over-dispersed ($>0$, off-manifold) at low guidance to under-dispersed ($<0$, collapse).}
\label{fig:dit-cfg}
\end{figure}

\FloatBarrier
\paragraph{SiT-XL/2: guidance transition in a second architecture.}
The same transition appears on the interpolant-based transformer
\textbf{SiT-XL/2} \citep{ma2024sit}, over six ImageNet classes with $n=500$ and
$\mathrm{CFG}\in\{1,2,4,8,16\}$. Across the six classes, the signed dispersion
readout indicates over-dispersion at $\mathrm{CFG}=1$ and under-dispersion at $\mathrm{CFG}\in\{4,8,16\}$; at $\mathrm{CFG}=2$,
five classes are under-dispersed and one is not assigned. At $\mathrm{CFG}=16$, FID exceeds its $\mathrm{CFG}=1$ value
on five classes but remains lower on class~910 ($107.2$ versus $125.8$; App.~\ref{app:dit-cfg},
Fig.~\ref{fig:sit-cfg}).
\paragraph{StyleGAN2-ADA: full-rank evaluation on FFHQ.}
We examine the collapse diagnosis outside diffusion models and ImageNet by sweeping \textbf{StyleGAN2-ADA}
(StyleGAN2: \citealp{karras2020analyzing}; ADA: \citealp{karras2020training}) on \textbf{FFHQ}
\citep{karras2019stylegan} over truncation $\psi\in\{1.0,0.9,0.7,0.5\}$
($n{=}5000{>}d$, giving a full-rank sample covariance and reduced
finite-sample bias). Table~\ref{tab:stylegan-truncation} summarizes the sweep.

FID, KID, and \refrev{Gaussian-kernel MMD} all increase monotonically as truncation tightens,
correctly ordering this sweep. ZID likewise separates every truncated setting from the untruncated generator reference
and additionally identifies under-dispersion, with a signed readout whose magnitude increases as
$\psi$ falls. Precision and Density instead increase as Recall declines, reflecting their emphasis on
different aspects of the precision--recall tradeoff.
In this full-rank FFHQ setting, FID, KID, \refrev{Gaussian-kernel MMD}, and the ZID score therefore agree on the
truncation ordering, while ZID additionally supplies calibrated detection and direction.

\begin{table}[H]\centering\small
\caption{\textbf{StyleGAN2-ADA truncation on FFHQ.}
Columns compare independent draws at the indicated $\psi$ with an untruncated generator reference;
the $\psi=1$ versus $\psi=1$ baseline uses two independent untruncated draws. FID and ZID use $5000$ samples; KID,
\refrev{Gaussian-kernel MMD}, Precision, Density, and Recall use a fixed $2500$-sample subset. Arrows give the conventional preferred direction. Within each metric row,
bold marks the best value for that metric. The ZID test $p$-value and $D_{\mathrm{ZID}}$ are diagnostic quantities rather than
ranking criteria and are not bolded. Larger ZID scores indicate stronger departures, while
negative $D_{\mathrm{ZID}}$ indicates under-dispersion; the $\psi=1$ versus $\psi=1$ baseline has no assigned direction.}
\label{tab:stylegan-truncation}
\setlength{\tabcolsep}{7pt}
\begin{tabular}{lrrrr}
\toprule
metric & \shortstack{$\psi=1$ vs.\\$\psi=1$ baseline} & $\psi=.9$ & $\psi=.7$ & $\psi=.5$ \\
\midrule
FID $\downarrow$       & \textbf{5.97} & 6.97 & 22.13 & 57.89 \\
KID $\downarrow$       & $\mathbf{-4.4{\times}10^{-5}}$ & .00119 & .01461 & .04596 \\
\refrev{Gaussian-kernel MMD} $\downarrow$ & \refrev{$\mathbf{-4.0\times10^{-5}}$} & \refrev{.00130} & \refrev{.01664} & \refrev{.05510} \\
Precision $\uparrow$   & .817 & .876 & .931 & \textbf{.963} \\
Density $\uparrow$     & .964 & 1.169 & 1.604 & \textbf{1.944} \\
Recall $\uparrow$      & \textbf{.750} & .698 & .506 & .221 \\
\midrule
ZID score $\downarrow$     & \textbf{.52} & 626.6 & $1.35{\times}10^5$ & $1.41{\times}10^6$ \\
ZID test $p$-value & .502 & .002 & .002 & .002 \\
$D_{\mathrm{ZID}}$ & undefined & $-22.7$ & $-441.0$ & $-1215.6$ \\
\bottomrule
\end{tabular}
\end{table}

\paragraph{What ZID adds across generators.}
Across the five generator families, ZID jointly provides scalar ordering,
calibrated detection, and directional diagnosis. On BigGAN and StyleGAN2-ADA, the ZID score increases
as truncation tightens and its signed readout identifies under-dispersion; conventional distances respond
as well but do not supply direction. On DiT and SiT, the ZID score measures departure on both sides of
the closest sampled guidance setting, while its signed readout distinguishes low-guidance over-dispersion
from high-guidance collapse. On DDPM, ZID reports under-dispersion at five steps and retains the member
signs when the fixed-bank comparison gives conflicting directions at later steps.

\FloatBarrier
\section{Gaming FID in pixel space}\label{sec:pixelgaming}
Proposition~\ref{prop:matched-moment} states the matched-moment blind spot in
feature space; here we realize it by optimizing actual pixels. Starting from noise we optimize a set of $N{=}512$ images at $64{\times}64$ by gradient
descent to match an ImageNet reference's Inception mean and covariance.
\refrev{The pretrained feature extractor is held fixed and only the image pixels are updated.}

This optimization directly targets the moments that determine FID. It lowers
FID to $24.7$, below the real--real baseline of $58.6$, while the resulting images remain visually
unrecognizable (Fig.~\ref{fig:pixelgaming}).
On the same saved feature banks, the real--holdout and real--gamed KID estimates are
$8.6{\times}10^{-6}$ and $-5.2{\times}10^{-5}$, respectively (App.~\ref{app:feature-gaming}).
Against the same reference and using 499 outer permutations, the held-out-real comparison yields a ZID
score of $0.86$ ($p{=}.328$), whereas the gamed-set comparison yields $17.2$, larger than all 499 relabeled scores.
For the gamed pair, ZID reports component magnitudes $S_W{=}17.9$ and
$S_D{=}9.8$; $S_D$ also exceeds all 499 relabeled values. The RISE dispersion arm is not
reference-active ($Z_D^{\rm RISE}{=}.58$, $p{=}.565$). The two active GPK arms have opposite signs
($Z_D^{\rm med}{=}{-}3.5$, $Z_D^{\rm small}{=}+4.3$), so $D_{\mathrm{ZID}}$ is undefined and the
dispersion diagnosis is \emph{member-sign conflict}.

Because the $512$ images are jointly optimized against the same reference
used for evaluation, the generated batch is reference-adaptive rather than an i.i.d.\ draw independent
of that reference. For this fixed pair, none of the 499 random outer
relabelings produces an equally or more extreme score; the plus-one Monte Carlo tail is therefore
$.002$. Because the pair is reference-adaptive, this tail is descriptive, not a calibrated population
two-sample $p$-value. Thus moment matching can drive FID below its real-data baseline for visually
unrecognizable samples, while ZID assigns this FID-targeted construction a large departure score under the same
reference.

A separate CIFAR-10 replication with the same $N{=}512$ optimization drives the
gamed FID to $31.1$, below the $76.0$ real--real baseline, while
the ZID score exceeds all 499 randomly relabeled scores \refrev{(App.~\ref{app:feature-gaming})}.

\enlargethispage{\baselineskip}
\section{Conclusion}\label{sec:conclusion}
FID's compactness is also its limitation: reducing a feature distribution to
its mean and covariance can make distinct distributions indistinguishable and cannot encode the direction
of a diversity change. ZID separates three tasks that a single scalar is often asked to serve. Its six-arm
score orders settings by severity within a fixed protocol, its outer-permutation $p$-value tests
distributional equality, and its component readouts characterize location-sensitive joint departure and
report a dispersion direction when supported.

Across controlled feature departures, real-feature replications, and
evaluations of pretrained generative models, these outputs provide broad detection, track progression as severity increases, and distinguish
low-guidance over-dispersion from high-guidance collapse. Retained member signs also reveal when
dispersion direction differs across scales. The evidence therefore
supports pairing calibrated evidence of distributional difference with diagnostic readouts rather than
relying on an unsigned moment summary alone.

\bibliographystyle{plainnat}
\bibliography{refs}

\section*{Use of AI assistance}
A large language model was used as a coding and writing assistant under the author's direction; all methods, code, and numerical results were verified by the author, who takes full responsibility for all claims.

\appendix

\section{Foundations and controlled validation}\label{app:foundations}

\subsection{Matched-moment feature controls}\label{app:feature-gaming}
{
Table~\ref{tab:feature-gaming} tests whether the feature-space FID inversion in the Introduction is
specific to one dataset or embedding. Each comparison uses $m=n=2500$ samples after projection to the
leading 128 pooled principal components. The Gaussian construction is drawn from the Gaussian fitted
to the reference sample.

\begin{table}[ht]\centering\small
\caption{\textbf{FID under fitted-Gaussian feature controls.} Lower values are nominally better.
The real--real column compares the reference with a held-out real sample; the Gaussian column replaces
the held-out sample by a draw from the fitted Gaussian. DINO denotes the
DINO ViT-S/16 CLS embeddings ($d=384$) defined in Sec.~\ref{sec:real-embedding-transfer}.}
\label{tab:feature-gaming}
\begin{tabular}{lrr}
\toprule
Dataset / embedding & real--real & fitted Gaussian \\
\midrule
CIFAR-10 / Inception & 2.517 & 1.157 \\
Tiny-ImageNet / Inception & 2.884 & 1.291 \\
CIFAR-10 / DINO & 76.472 & 33.864 \\
Tiny-ImageNet / DINO & 113.849 & 53.552 \\
\bottomrule
\end{tabular}
\end{table}

Table~\ref{tab:matched-bimodal-values} restores the reference sample's empirical mean and covariance
exactly. By contrast, the large-sample comparison in Sec.~\ref{sec:signed-imbalance} draws from a fixed population-matched alternative, so its empirical
FID follows the real--real finite-sample baseline rather than being numerically zero. In that comparison,
empirical FID is evaluated through $n=25$k observations in each sample. The six-arm ZID comparison uses
$n=2$k, 20 repetitions, and 99 outer permutations.

{
\begin{table}[H]\centering\small
\caption{\textbf{Full-dimensional matched-bimodal feature controls.} Each comparison uses raw
Inception-2048 features with $m=n=2500$. The matched-bimodal sample is recolored to the reference
mean and covariance. KID is the unbiased degree-$3$ polynomial-kernel MMD estimate and may be
slightly negative.}
\label{tab:matched-bimodal-values}
\begin{tabular}{llrr}
\toprule
Dataset & Metric & real--real & matched bimodal \\
\midrule
CIFAR-10 & FID & 19.5146 & 0.0030 \\
         & KID & $9.5{\times}10^{-5}$ & $-3.20{\times}10^{-4}$ \\
Tiny-ImageNet & FID & 25.8680 & 0.0024 \\
              & KID & $-5.7{\times}10^{-5}$ & $-3.39{\times}10^{-4}$ \\
\bottomrule
\end{tabular}
\end{table}

The pixel-space stress test in Sec.~\ref{sec:pixelgaming} uses a separate pair of
$512\times2048$ feature banks. On these same banks, KID is $8.6{\times}10^{-6}$ for real versus held-out real
and $-5.2{\times}10^{-5}$ for real versus the optimized batch.

\refrev{The CIFAR-10 replication follows the same reference-adaptive stress-test design: a noise-initialized
image batch is optimized to match the feature mean and covariance of one real sample, while a disjoint
real sample supplies the real--real baseline. Because the optimized batch is adapted to its evaluation
reference, the resulting permutation tail is descriptive rather than a calibrated population two-sample
$p$-value.}
}
}
\FloatBarrier

\subsection{Proof of Proposition~\ref{prop:matched-moment}}\label{app:proof}
\begin{proof}
Let $G=\mathcal N(\mu_P,\Sigma_P)$. If $P\neq G$, take $Q=G$. It has the same mean and covariance as
$P$ and is a different distribution. It remains to handle $P=G$. Write the nonzero part of the spectral
decomposition as $\Sigma_P=U_r\Lambda_rU_r^\top$, where $r\ge1$, and let
$\varepsilon\in\{-1,+1\}^r$ have independent Rademacher coordinates. Define
$Q$ as the law of
\[
\mu_P+U_r\Lambda_r^{1/2}\varepsilon.
\]
This distribution has mean $\mu_P$ and covariance $\Sigma_P$, but it has finite support and therefore
differs from the Gaussian law $P=G$ (including when that Gaussian is supported on a proper subspace).
Thus in either case there is a $Q\neq P$ with matching first two moments, and
$M(P,Q)=0=M(P,P)$.
\end{proof}
\FloatBarrier

\subsection{Controlled-departure overview}
\label{app:controlled-departures}
\refrev{Each controlled comparison starts from two disjoint real-feature samples and applies one
transformation to the second sample. The eight transformations introduce, respectively, a leading-direction
location shift, isotropic contraction or expansion, linear dependence, nonlinear dependence with unchanged
marginals, skewness, kurtosis, moment-matched multimodality, or noise in low-variance coordinates. The
dependence and higher-order transformations are followed by affine moment restoration where required, so
their intended differences are not reducible to a change in feature mean or covariance. Figure~\ref{fig:dom}
uses one panel-selected signal per departure, held fixed across all methods; the ranking and severity studies
use ordered versions of the same transformations.}

\paragraph{Feature representations.}
\refrev{Inception evaluations use the 2048-dimensional penultimate Inception-v3 representation
\citep{szegedy2016inception}; DINO evaluations use ViT-S/16 CLS embeddings \citep{caron2021dino}.
The DDPM representation analysis uses DINOv2 rather than DINO. All use the corresponding standard
released weights and preprocessing. The pixel-space stress test and controlled feature-space studies use
separate standard Inception implementations, so absolute FID and KID values are interpreted only within
their stated protocols.}

\paragraph{Moment-restored resampling with perturbation.}
\refrev{Table~\ref{tab:memorize} begins with an independent real sample, replaces a specified fraction of
its observations by perturbed resamples from that sample, and then restores its empirical mean and
covariance. The replacement fraction controls the amount of resampling, while $\varepsilon$ controls the
perturbation scale relative to the coordinatewise standard deviations.}

\paragraph{Comparator settings.}
\refrev{MIND uses a fixed common set of random projections. Precision, Recall, Density, and Coverage use
a common nearest-neighbor specification. Vendi Score is computed from a cosine-similarity Gram matrix,
PERMDISP compares the samples' mean distances to their respective centroids, and Gaussian-kernel MMD
uses the unbiased $U$-statistic with a pooled median-distance bandwidth. Permutation-calibrated rows
reuse the same pooled samples and statistic-specific representation across relabelings.}

\paragraph{C2ST protocols.}\label{app:c2st-protocol}
\refrev{The controlled comparisons use a stratified half split, a fixed one-hidden-layer MLP trained on
one half, and held-out accuracy on the other. Calibration relabels the balanced held-out labels while
keeping the fitted classifier and its predictions fixed.}

\paragraph{ECS protocol.}\label{app:protocol-details}\label{app:ecs-sensitivity}
\refrev{ECS depends on a frequency parameter $T$: smaller values emphasize lower-order and location
differences, whereas larger values emphasize higher moments and tails. An independent preliminary study
selected the fixed value $T=.28$ to balance these sensitivities before the reported comparison. Because
coordinatewise ECS is not invariant to coordinate rescaling or rotation, its results are conditional on the
PCA coordinates used by the controlled comparison.}

\subsection{Outer-permutation calibration}\label{app:calibration}
Table~\ref{tab:typeI} evaluates the final six-arm equality test on
real-feature null comparisons across two datasets and two embeddings. Its rejection rates remain
near their nominal levels, supporting the outer-permutation calibration used in the main
experiments.

\begin{table}[ht]\centering\small
\caption{\textbf{Outer-permutation calibration on real features.}
This independent real-null study applies one common protocol across four
dataset--embedding settings: 500 repetitions per setting, $m=n=200$, and 99
permutations for the six-arm equality test. The CIFAR--Inception row repeats the Fig.~\ref{fig:dom} null
setting as an internal reference for the matrix; the main-text check uses 999 permutations for finer
$p$-value resolution. DINO denotes the DINO ViT-S/16 CLS
embeddings ($d=384$) defined in Sec.~\ref{sec:real-embedding-transfer}.
Rejection uses $p\le\alpha$; $p=.01$ is the smallest attainable value.}
\label{tab:typeI}
\begin{tabular}{lccc}
\toprule
dataset / embedding & $\alpha{=}.01$ & $.05$ & $.10$ \\
\midrule
CIFAR-10 / Inception      & .008 & .028 & .094 \\
CIFAR-10 / DINO           & .010 & .052 & .104 \\
Tiny-ImageNet / Inception & .010 & .054 & .098 \\
Tiny-ImageNet / DINO      & .014 & .048 & .090 \\
\bottomrule
\end{tabular}
\end{table}

\subsection{GPK, GET, and RISE constructions and sensitivity to arm-level tail probabilities}\label{app:controlled}
GPK, GET, and RISE instantiate within-sample similarity differently:
GPK uses dense Gaussian-kernel weights; GET uses a union of edge-disjoint minimum
spanning trees; and RISE uses a directed nearest-neighbor graph weighted by neighbor rank.
\refrev{The reported standalone analyses use fixed graph specifications and the pooled
median-distance GPK bandwidth.}
Thus GET and RISE differ in graph construction and weighting, whereas GPK is kernel based.
The standalone RISE and GET tests use
$T=Z_W^2+Z_D^2$ with their asymptotic $\chi^2_2$ references. For GPK,
$W-\mathbb E_0W$ is proportional to $\widehat{\mathrm{MMD}}_u^2$
(Sec.~\ref{sec:signed-imbalance}), and the standalone GPK-med and GPK-small tests calibrate their
quadratic statistics by label permutation rather than assigning a $\chi^2_2$ reference from coordinate
standardization. The normal-reference tail of an individual GPK $Z_D$ coordinate is an arm-level
quantity used in constructing ZID, not the standalone GPK omnibus $p$-value. These family-specific
calibrations apply only to these standalone rows; the final six-arm ZID $p$-value is the outer
permutation tail of the complete aggregated score.

To check sensitivity to the Gaussian-reference arm tails used in ZID, a
\refrev{nested comparison replaces each arm's two-sided normal tail by an inner-permutation estimate while
retaining outer-permutation calibration.}
The comparison uses 300 paired repetitions on each of the eight Fig.~\ref{fig:dom} departures; power
at the $0.05$ significance level is reported in
Table~\ref{tab:arm-tail-sensitivity}.
Among the 2,400 paired alternatives, 143 rejection decisions occur only
with Gaussian-reference tails and 53 only with permutation tails. The departure-wise probabilities that
an alternative score exceeds its paired null score range from $.947$ to $.997$ and from $.933$ to
$.997$, respectively; the pooled probabilities are $.964$ and $.958$. Across the 2,400 paired null
comparisons, the corresponding rejection rates are $.050$ and $.037$. The two
methods for assigning arm-level tail probabilities therefore yield similar probabilities that an alternative score exceeds its paired null score,
while the Gaussian-reference implementation has somewhat higher detection power in this finite-permutation
comparison.

\begin{table}[!htbp]
\centering
\small
\caption{\textbf{Sensitivity to arm-level tail probabilities.}
Each cell reports power at the $0.05$ significance level over 300 paired repetitions.
The Gaussian-reference specification uses the two-sided normal-reference tail of each standardized arm;
the permutation-tail specification estimates each arm tail by label permutation.}
\label{tab:arm-tail-sensitivity}
\setlength{\tabcolsep}{9pt}
\begin{tabular}{lcc}
\toprule
departure & Gaussian-reference & permutation-tail \\
\midrule
location & .823 & .727 \\
dispersion & .763 & .763 \\
linear dependence & .707 & .643 \\
nonlinear dependence & .977 & .957 \\
skewness & .827 & .790 \\
kurtosis & .840 & .807 \\
off-manifold support & .760 & .777 \\
matched-moment multimodality & .727 & .660 \\
\bottomrule
\end{tabular}
\end{table}

\subsection{High-dimensional dispersion sensitivity and coordinate separation}\label{app:power}
We compare median-bandwidth GPK, KID, and FID under one aligned protocol.
Figure~\ref{fig:finegrid} compares power across dimension and sample size. The
variance-collapse amount $s$ multiplies standardized deviations from the reference mean by $1-s$, so
$s=0$ is the null.

\begin{figure}[!t]\centering
\includegraphics[width=\textwidth]{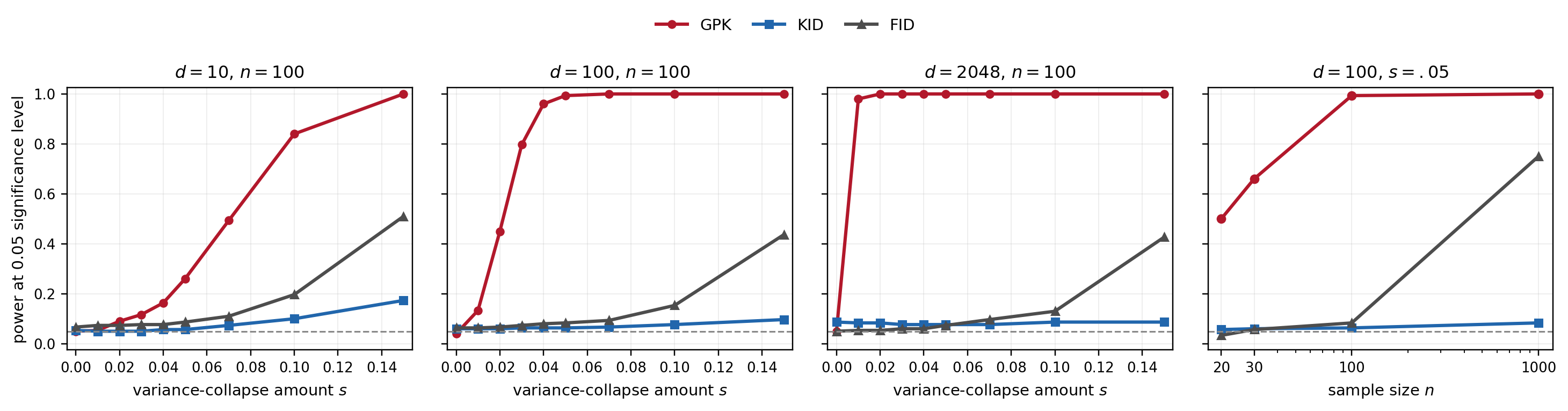}
\caption{\textbf{Aligned sensitivity to Gaussian shrinkage across
dimension and sample size.} Empirical power for median-bandwidth GPK, degree-3 KID,
and FID; $s=0$ denotes the null.}
\label{fig:finegrid}
\end{figure}

Within each repetition, the three methods use the same Gaussian
reference--alternative pair and the same 99 label relabelings; power at the $0.05$ significance level is
estimated from 300 repetitions. The first three panels fix $n=100$ and vary dimension over the full shrinkage grid;
the fourth fixes $d=100$ and $s=.05$ and varies sample size. At $n=100$, GPK power for a $1\%$
shrinkage is $.98$ at $d=2048$; at $d=10$, its power is $.493$ for a $7\%$ shrinkage and $.84$ for a
$10\%$ shrinkage. At $d=100$ and $5\%$ shrinkage, GPK power is $.50$, $.66$, $.993$, and $1.00$ at
$n=20,30,100,$ and $1000$, respectively. FID reaches $.75$ at $n=1000$, while KID remains at $.083$.
Across the displayed range, GPK responds sharply to small shrinkages as dimension or sample size increases,
KID remains below $.20$, and FID rises more gradually.

\section{Protocols and supporting analyses for pretrained generative models}\label{app:real-details}

\subsection{Sampling and calibration for pretrained generative models}
\label{app:generator-protocol}
\refrev{Section~\ref{sec:real} evaluates released generative models under fixed model weights, feature
representations, sampling pipelines, and reference sets. Within each sweep, only the sampling or truncation
parameter listed in Table~\ref{tab:external-validity} changes. Each entry is a fixed reference--model comparison
rather than a power estimate over independently trained or generated models. All ZID equality-test and
$D$-component $p$-values use the
same six-arm construction and 499 outer relabelings.}

\refrev{The BigGAN comparison uses ten ImageNet classes with $m=n=450$. The CIFAR DDPM comparison
uses $m=n=200$ at 5, 15, and 25 sampling steps. DiT-XL/2 and SiT-XL/2 use six ImageNet classes with
$m=n=500$ at each guidance setting.
StyleGAN2-ADA compares independent $n=5000$ draws across the displayed truncation values; the remaining
metrics use a common fixed subset. Inception and DINOv2 results use their respective standard released
representations and preprocessing.}

\subsection{BigGAN truncation control}\label{app:biggan}
BigGAN-deep-128 provides a control in which tighter truncation should
reduce diversity. FID is lowest at $\psi=.8$ and rises from $\psi=.8$ to $.4$
as truncation tightens. Over truncation $\psi$ ($m{=}n{=}450$,
$10$ ImageNet classes), the ZID score
rises from $328$ to $1150$ as truncation tightens, while its aggregated $S_D$ rises
from $42.2$ to $176.7$ and the dispersion diagnosis remains under-dispersion. At $m{=}n{=}450{<}d{=}2048$, the real--real FID is $70.3$, and the
generator--real values are $89$--$102$ (Table~\ref{tab:biggan}).

\begin{table}[H]\centering\small\setlength{\tabcolsep}{4pt}
\caption{\textbf{BigGAN truncation control.}
BigGAN-deep-128 versus real ImageNet from the same $10$ classes
(Inception-2048, $m{=}n{=}450$). The dispersion ratio is
$\operatorname{tr}(\widehat\Sigma_{\rm gen})/\operatorname{tr}(\widehat\Sigma_{\rm real})$.
For the generator rows, $(\widehat p_{\mathrm{ZID}},\widehat p_D)=(.002,.002)$; for real--real,
the pair is $(.598,.856)$.}
\label{tab:biggan}
\begin{tabular}{lcccccccc}
\toprule
$\psi$ & disp.\ ratio & FID & KID & ZID score & $S_W$ & $S_D$ & $D_{\mathrm{ZID}}$ & \shortstack{dispersion\\diagnosis} \\
\midrule
1.0 & 0.82 & 90.0 & .0075 & 328 & 329.1 & 42.2 & $-42.2$ & under-dispersion \\
0.8 & 0.79 & 88.7 & .0108 & 566 & 566.4 & 96.1 & $-96.1$ & under-dispersion \\
0.6 & 0.77 & 96.8 & .0145 & 813 & 813.8 & 119.3 & $-119.3$ & under-dispersion \\
0.4 & 0.76 & 101.6 & .0173 & 1150 & 1151.2 & 176.7 & $-176.7$ & under-dispersion \\
\midrule
real--real & 1.00 & 70.3 & $-.0004$ & .353 & 1.05 & .12 & undefined & not assigned \\
\bottomrule
\end{tabular}
\end{table}
\FloatBarrier

\subsection{Detailed CFG sweeps: DiT and SiT}\label{app:dit-cfg}
Table~\ref{tab:dit-cfg-full} gives the class-specific DiT-XL/2 guidance
results underlying the aggregate main-text view in Fig.~\ref{fig:dit-cfg}.
\begin{table}[H]\centering\footnotesize
\setlength{\tabcolsep}{3.5pt}
\caption{\textbf{Per-class DiT-XL/2 classifier-free-guidance sweep.} Six ImageNet classes with class-matched real images (Inception-2048, $m=n=500$). Parentheses report $\widehat p_D$; ``over'' and ``under'' abbreviate the dispersion diagnoses. All $\widehat p_{\mathrm{ZID}}=.002$ with 499 permutations.}
\label{tab:dit-cfg-full}
\begin{tabular}{lrcccccccc}
\toprule
class & CFG & FID & $S_{\mathrm{ZID}}$ & KID & Prec & Recall & $S_W$ & $D_{\mathrm{ZID}}$ & diagnosis ($\widehat p_D$) \\
\midrule
10 & 1 & 40.9 & 4.1e3 & 0.0149 & 0.652 & 0.805 & 4.1e3 & 41.1 & over (.002) \\
 & 2 & 20.2 & 3.2e3 & 0.0053 & 0.890 & 0.639 & 3.2e3 & -50.3 & under (.002) \\
 & 4 & 27.7 & 26e3 & 0.0121 & 0.814 & 0.251 & 26e3 & -100.2 & under (.002) \\
 & 8 & 38.4 & 69e3 & 0.0216 & 0.698 & 0.065 & 69e3 & -148.3 & under (.002) \\
\midrule
190 & 1 & 42.3 & 2.2e3 & 0.0128 & 0.698 & 0.745 & 2.2e3 & 18.4 & over (.002) \\
 & 2 & 35.4 & 6.8e3 & 0.0140 & 0.868 & 0.500 & 6.8e3 & -61.3 & under (.002) \\
 & 4 & 53.1 & 33e3 & 0.0320 & 0.814 & 0.179 & 33e3 & -83.8 & under (.002) \\
 & 8 & 59.3 & 60e3 & 0.0398 & 0.782 & 0.109 & 60e3 & -141.8 & under (.002) \\
\midrule
370 & 1 & 42.6 & 3.9e3 & 0.0153 & 0.624 & 0.773 & 3.9e3 & 30.8 & over (.002) \\
 & 2 & 37.4 & 8.2e3 & 0.0176 & 0.818 & 0.532 & 8.2e3 & -66.9 & under (.002) \\
 & 4 & 56.2 & 54e3 & 0.0361 & 0.868 & 0.195 & 54e3 & -175.7 & under (.002) \\
 & 8 & 64.0 & 81e3 & 0.0408 & 0.892 & 0.025 & 81e3 & -223.4 & under (.002) \\
\midrule
550 & 1 & 80.9 & 6.9e3 & 0.0374 & 0.498 & 0.739 & 6.9e3 & 39.2 & over (.002) \\
 & 2 & 55.2 & 4.7e3 & 0.0221 & 0.850 & 0.559 & 4.7e3 & -43.4 & under (.002) \\
 & 4 & 87.0 & 45e3 & 0.0622 & 0.938 & 0.285 & 45e3 & -150.0 & under (.002) \\
 & 8 & 98.8 & 65e3 & 0.0716 & 0.936 & 0.152 & 65e3 & -184.4 & under (.002) \\
\midrule
730 & 1 & 85.6 & 12e3 & 0.0375 & 0.514 & 0.732 & 12e3 & 30.6 & over (.002) \\
 & 2 & 55.9 & 2.5e3 & 0.0152 & 0.818 & 0.558 & 2.5e3 & -40.3 & under (.002) \\
 & 4 & 87.8 & 35e3 & 0.0577 & 0.810 & 0.268 & 35e3 & -124.1 & under (.002) \\
 & 8 & 121.2 & 83e3 & 0.0926 & 0.702 & 0.123 & 83e3 & -161.9 & under (.002) \\
\midrule
910 & 1 & 121.0 & 16e3 & 0.0467 & 0.410 & 0.731 & 16e3 & 44.7 & over (.002) \\
 & 2 & 76.4 & 0.6e3 & 0.0087 & 0.692 & 0.535 & 0.6e3 & -5.9 & under (.008) \\
 & 4 & 92.0 & 12e3 & 0.0322 & 0.812 & 0.342 & 12e3 & -71.5 & under (.002) \\
 & 8 & 121.3 & 35e3 & 0.0567 & 0.772 & 0.111 & 35e3 & -128.3 & under (.002) \\
\bottomrule
\end{tabular}
\end{table}

Figure~\ref{fig:sit-cfg} summarizes the \refrev{corresponding guidance transition} on \textbf{SiT-XL/2}
\citep{ma2024sit}, an interpolant transformer with a continuous flow objective. The sweep uses
the same six classes and $n{=}500$, with $\mathrm{CFG}\in\{1,2,4,8,16\}$. FID,
KID, and the \refrev{Frobenius covariance discrepancy} are minimized near $\mathrm{CFG}{=}2$ and increase on
both sides, but their scalar values do not encode which side is over- versus under-dispersed
(FID on class~$10$:
$41\!\to\!22\!\to\!34\!\to\!43\!\to\!49$). The \refrev{covariance-trace difference,
centroid-distance difference (PERMDISP), and Vendi Score difference} corroborate that dispersion changes across the sweep.
\refrev{For reference sample $X$ and generator sample $Y$, the covariance-trace difference is $\sum_k\{\widehat{\operatorname{Var}}(Y_k)-\widehat{\operatorname{Var}}(X_k)\}$, the Vendi Score difference is $V(Y)-V(X)$, and the centroid-distance difference is $\bar r_Y-\bar r_X$, where $\bar r_X=n^{-1}\sum_i\lVert X_i-\bar X\rVert$ and likewise for $Y$. The Frobenius covariance discrepancy is $\lVert\widehat\Sigma_Y-\widehat\Sigma_X\rVert_F$.}
ZID's final aggregate readout supplies the direction: on class~$10$, its
$D$-component magnitude $S_D$ grows from $42.3$ at CFG$=1$ to $176.6$ at CFG$=16$, while the
direction changes from over-dispersion to under-dispersion.

\begin{figure}[H]\centering
\includegraphics[width=\linewidth]{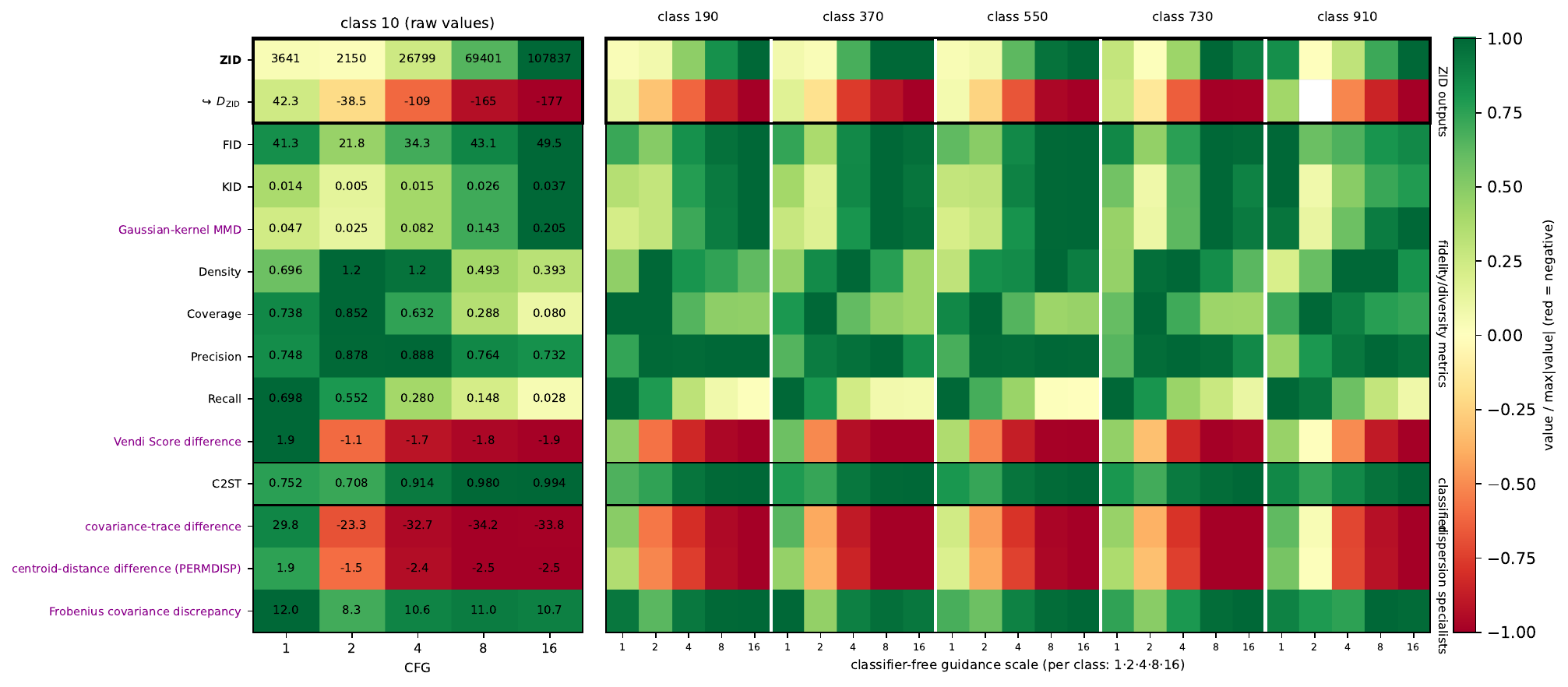}
\caption{\textbf{Classifier-free-guidance corroboration on SiT-XL/2.} Classifier-free-guidance
sweep vs.\ real ImageNet for six classes ($n{=}500$), using the \refrev{14 displayed quantities}. \textbf{Class~10 (left) is shown with its
raw values}; the other five classes are shown as normalized color cells. Color is each row scaled by its own
peak $|{\cdot}|$ within a class (diverging; green${>}0$, \textbf{red${<}0$}), so color shows each method's
response shape, not cross-row magnitudes (the raw values do). The boxed first two rows place ZID and
its signed dispersion readout $D_{\mathrm{ZID}}$ together; shared external methods below follow their relative order in
Fig.~\ref{fig:dom}. FID, KID, and the \refrev{Frobenius covariance discrepancy} attain their minima near $\mathrm{CFG}{=}2$.
The specialist rows corroborate a dispersion change, while $D_{\mathrm{ZID}}$
flips positive$\,\to\,$negative (the red band) and identifies the high-guidance side as
under-dispersed; the blank cell is the one case in which no direction is assigned.}
\label{fig:sit-cfg}
\end{figure}

\FloatBarrier

\section{Additional aggregation, ranking, and bandwidth controls}
\label{app:frozen-design-controls}

{
These controls retain the six-arm definition and its
$0.175\sigma_{\rm med}$ \refrev{GPK-small} member. Unless stated otherwise, they use the eight
controlled departures, CIFAR-Inception PCA-128 features, and $n=200$ observations per sample;
each table states its repetition and permutation budgets. The ranking controls use the full-range
sweeps from Sec.~\ref{sec:ranking}; the power controls use the departure definitions and signal
levels from Fig.~\ref{fig:dom}, with independent samples.

\paragraph{Aggregation rule.}
Table~\ref{tab:aggregation-control} compares five ways to aggregate the same six standardized arms.
For the power columns, every rule shares samples, arm matrices, and 199 outer relabelings within a
replicate. Flat Simes has the highest mean power and the strongest ranking summaries across the full severity ladders,
whereas Cauchy has a one-point higher minimum power estimate. The differences among flat Simes, max,
and Cauchy are descriptively small, while sum-of-squares and Fisher are weaker in the displayed power
cells. Taken together, these results place flat Simes in the leading group
across detection and ranking summaries, supporting its use as ZID's aggregation rule.

\begin{table}[!htbp]\centering\small
\caption{\textbf{Paired aggregation-rule comparison.} Mean and minimum
summaries are across the eight controlled departures. Power and null
rejection are evaluated at the $0.05$ significance level using 100 repetitions
and 199 outer permutations per departure; ranking uses 100 repetitions per
sweep level and per mild-versus-final pair.
All rules are evaluated on identical arms within each replicate.}
\label{tab:aggregation-control}
\setlength{\tabcolsep}{5pt}
\begin{tabular}{lccccccc}
\toprule
rule & \multicolumn{2}{c}{power} & null & \multicolumn{2}{c}{Spearman $\rho$} &
\multicolumn{2}{c}{paired accuracy} \\
\cmidrule(lr){2-3}\cmidrule(lr){5-6}\cmidrule(lr){7-8}
 & mean & min & mean & mean & min & mean & min \\
\midrule
flat Simes & .826 & .67 & .041 & .829 & .562 & .968 & .83 \\
max $|Z_k|$ & .818 & .66 & .042 & .828 & .556 & .966 & .82 \\
sum of squares & .784 & .60 & .048 & .822 & .543 & .958 & .80 \\
Fisher & .775 & .58 & .055 & .812 & .535 & .956 & .79 \\
Cauchy & .825 & .68 & .042 & .818 & .539 & .963 & .80 \\
\bottomrule
\end{tabular}
\end{table}

\paragraph{Expanded ranking across increasing severity.}
Table~\ref{tab:expanded-ranking} applies the same full-range ranking
protocol to the four standalone tests and five external comparators: FID, KID, multiscale MMD,
MIND, and ECS. \refrev{Multiscale MMD uses a fixed sum of Gaussian kernels spanning several scales.}
Relative to these external comparators, ZID ranks first on all four summaries: mean and
minimum Spearman correlations across the eight departures, and mean and minimum pairwise ordering accuracies. Among
these standalone tests, RISE ties ZID in mean pairwise ordering accuracy ($.969$ versus $.968$) and has
higher minimum pairwise ordering accuracy ($.90$ versus $.83$), whereas ZID has higher mean and
minimum Spearman correlations. This contrast separates monotonicity across the full severity ladders from ordering
the mild and final levels.
Density is not included in Table~\ref{tab:expanded-ranking} because converting its directional output into a
two-sided departure score requires an additional orientation choice.

\begin{table}[!htbp]\centering\small
\caption{\textbf{Expanded ranking across increasing severity.}
Spearman correlation uses the complete six-level sweep for each
departure family; paired accuracy compares
100 independent mild-versus-final pairs per
family. Mean and minimum
are across the same eight departure families. \refrev{MIND uses common random
projections within each paired replicate, and ECS uses the independently selected fixed specification.}
Rows are grouped as the complete ZID construction, the standalone
RISE, GPK-med, GPK-small, and GET tests, and external comparators.}
\label{tab:expanded-ranking}
\setlength{\tabcolsep}{7pt}
\begin{tabular}{lcccc}
\toprule
method & mean $\rho$ & min $\rho$ & mean paired accuracy & min paired accuracy \\
\midrule
ZID & .829 & .562 & .968 & .83 \\
\midrule
RISE & .808 & .494 & .969 & .90 \\
\refrev{GPK-med} & .411 & $-.082$ & .752 & .47 \\
\refrev{GPK-small} & .556 & .087 & .821 & .55 \\
GET & .679 & .199 & .906 & .63 \\
\midrule
FID & .266 & $-.406$ & .659 & .27 \\
KID & .378 & .010 & .741 & .48 \\
MMD$_{\rm ms}$ & .597 & .043 & .875 & .57 \\
MIND & .529 & $-.011$ & .839 & .51 \\
ECS & .380 & $-.030$ & .776 & .55 \\
\bottomrule
\end{tabular}
\end{table}

\paragraph{GPK-small bandwidth sensitivity.}
Within each replicate, Table~\ref{tab:bandwidth-control} holds samples and
\refrev{relabelings fixed and changes only the GPK-small multiplier} over seven prespecified
values from $.10$ to $.25$. Across this range, minimum power varies from $.70$ to $.71$ and mean null
rejection from $.035$ to $.048$. The neighboring $.15$ and $.20$ ratios differ from $.175$ by at most
$.05$ in power on every departure and have minimum paired decision
agreement of at least $.95$ under the alternative and $.98$ under the null. Wider changes produce larger
power differences or lower paired agreement on at least one departure. These quantities are descriptive
sensitivity summaries rather than a rule for selecting the default.

\begin{table}[!htbp]\centering\small

\caption{\textbf{Seven-point GPK-small bandwidth profile.} Only the \refrev{GPK-small} multiplier
changes; the other four arms, samples, and 199 outer relabelings are shared across ratios within each
of 100 repetitions per departure. Entries give the minimum power and mean null rejection rate across
departures at the $0.05$ significance level. The final three columns compare each ratio with $.175$:
the largest absolute departure-wise power difference and the minimum paired decision agreement under
the alternative and null.}
\label{tab:bandwidth-control}
\setlength{\tabcolsep}{3pt}
\begin{tabular}{lrrrrr}
\toprule
multiplier & min. power & mean null & max $|\Delta|$ & min. alt. agree & min. null agree \\
\midrule
$.100$ & .70 & .035 & .070 & .90 & .94 \\
$.125$ & .70 & .041 & .050 & .91 & .97 \\
$.150$ & .71 & .039 & .050 & .95 & .98 \\
$.175$ (reference) & .70 & .039 & .000 & 1.00 & 1.00 \\
$.200$ & .70 & .043 & .040 & .96 & .98 \\
$.225$ & .70 & .048 & .050 & .93 & .98 \\
$.250$ & .70 & .044 & .050 & .95 & .96 \\
\bottomrule
\end{tabular}
\end{table}
}
\FloatBarrier

\section{Detection across severity, dimension, and sample size}\label{app:robustness-details}

\subsection{Detection across increasing severity}\label{app:severity}
To examine how ZID's detection power changes as each departure strengthens,
we extend the single signal level used in Fig.~\ref{fig:dom} to a five-level ladder.
Each level uses 300 repetitions and 499 outer permutations per cell. Seven departures use multiples $\{0,.5,.75,1,1.25\}$ of their
Fig.~\ref{fig:dom} signal. For matched-moment multimodality, the ladder varies the fraction
$\lambda\in\{0,.25,.5,.75,1\}$ of initially real observations replaced by their counterparts from the same
moment-matched bimodal construction. This parameter ranges from the unmodified real sample at $\lambda=0$ to the
full bimodal alternative at $\lambda=1$.
Figure~\ref{fig:severity-suite} shows monotonically increasing empirical
power across all eight departures.

\begin{figure}[H]\centering
\includegraphics[width=.88\textwidth]{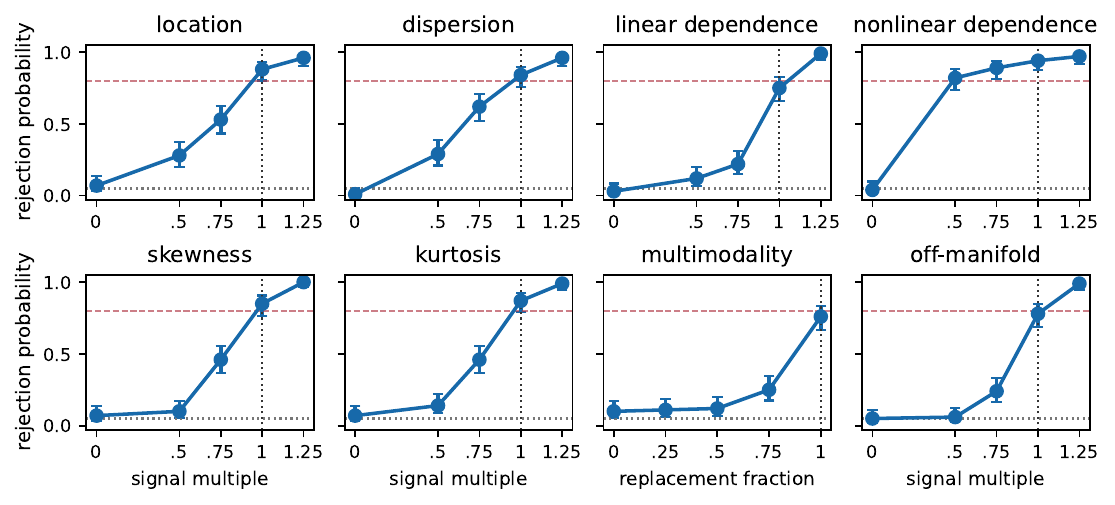}
\caption{\textbf{Detection across increasing severity.} Empirical rejection rates use 300 repetitions
and 499 outer permutations per cell at the $0.05$ significance level; bars are Wilson $95\%$ intervals.
Horizontal dotted and dashed lines mark $.05$ and $.80$, and the vertical dotted line marks the
Fig.~\ref{fig:dom} signal. For multimodality the horizontal coordinate is the fraction of initially real observations
replaced by their moment-matched bimodal counterparts; for the other seven departures it is a multiple
of the Fig.~\ref{fig:dom} signal.}
\label{fig:severity-suite}
\end{figure}

\subsection{Dimension--sample-size grid}\label{app:domgrid}
We compare methods across $d\in\{128,512,2048\}$ and $n\in\{50,200,600\}$ for the eight
controlled-departure constructions in App.~\ref{app:controlled-departures}.
\refrev{Figure~\ref{fig:domrobust} reports power for the aligned 17-row method panel. A panel-level
criterion selects one signal per cell to limit widespread ceiling effects, and that signal is held fixed
across methods; comparisons are therefore within cells rather than trends over $d$ or $n$.}

\begin{figure}[H]\centering
\includegraphics[page=1,width=0.94\linewidth]{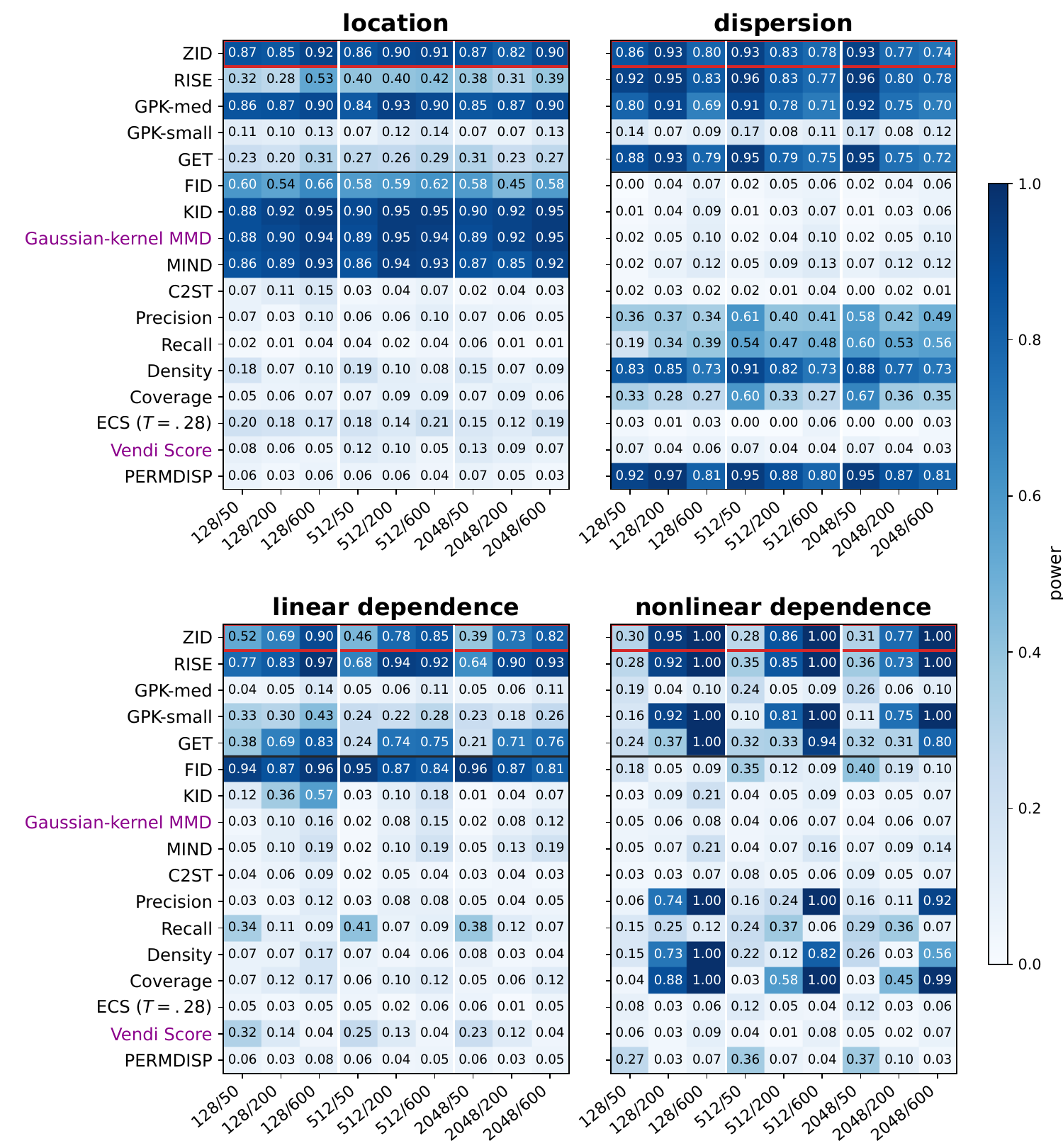}
\caption{\textbf{Detection across departures, sample sizes, and dimensions.}
The \refrev{17 rows} match Fig.~\ref{fig:dom}; columns are nine $(d,n)$ cells and panels are eight
controlled departures. Each cell uses CIFAR-10 Inception features projected to $d$ pooled principal
components, 100 repetitions at the $0.05$ significance level, and one signal shared by all methods;
permutation-calibrated rows use 99 relabelings. The boxed row is ZID; \refrev{GPK-small} uses
$0.175\sigma_{\rm med}$. Right panels share the method ordering shown at left.}
\label{fig:domrobust}
\end{figure}
\begin{figure}[p]\centering
\textit{Figure~\ref{fig:domrobust} continued.}\par\smallskip
\includegraphics[page=2,width=0.94\linewidth]{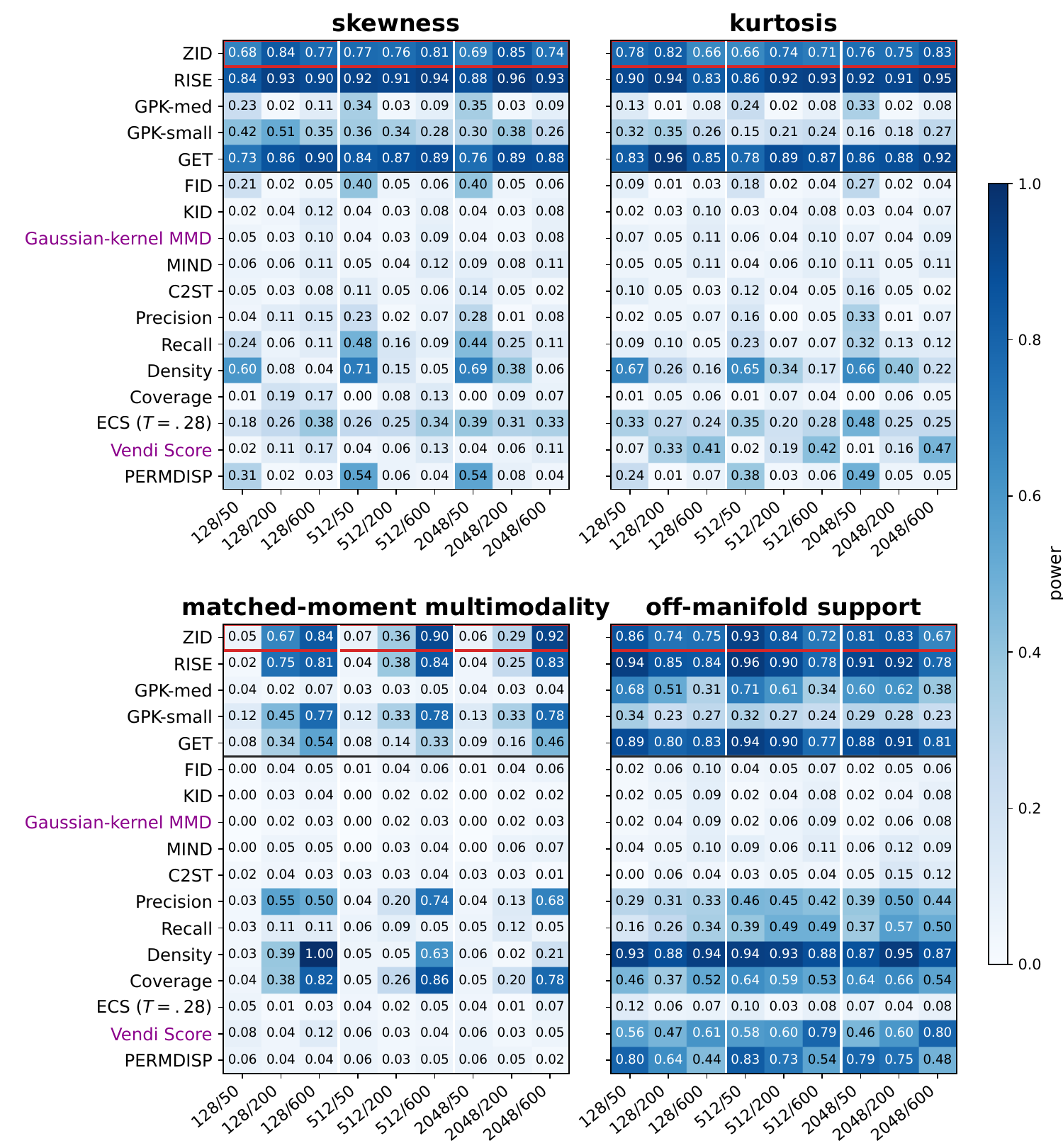}
\end{figure}
\FloatBarrier
\refrev{Table~\ref{tab:domnull} reports the maximum panel-specific null rejection estimate for each
method and $(d,n)$ cell.}
\begin{table}[H]\centering\small\setlength{\tabcolsep}{4pt}
\caption{\textbf{Null rejection across the aligned robustness grid.} The \refrev{17 rows} match Fig.~\ref{fig:domrobust}. For each method and $(d,n)$ cell, the table reports the maximum of the eight panel-specific null rejection-rate estimates, providing a worst-panel calibration summary. Each estimate uses real CIFAR-Inception features and 100 repetitions at the $0.05$ significance level; rejection is defined by $p\le.05$, and permutation-calibrated rows use 99 relabelings. The displayed values range from $.00$ to $.10$; the binomial s.e.\ of each underlying estimate is approximately $.022$. \refrev{GPK-small} uses $0.175\sigma_{\rm med}$.}
\label{tab:domnull}
\begin{tabular}{lccccccccc}
\toprule
method & \multicolumn{3}{c}{$d{=}128$} & \multicolumn{3}{c}{$d{=}512$} & \multicolumn{3}{c}{$d{=}2048$} \\
\cmidrule(lr){2-4}\cmidrule(lr){5-7}\cmidrule(lr){8-10}
 & $n{=}50$ & $n{=}200$ & $n{=}600$ & $n{=}50$ & $n{=}200$ & $n{=}600$ & $n{=}50$ & $n{=}200$ & $n{=}600$ \\
\midrule
ZID & 0.05 & 0.03 & 0.01 & 0.03 & 0.05 & 0.02 & 0.03 & 0.05 & 0.02 \\
RISE & 0.07 & 0.03 & 0.05 & 0.05 & 0.05 & 0.04 & 0.03 & 0.04 & 0.04 \\
\refrev{GPK-med} & 0.04 & 0.05 & 0.04 & 0.01 & 0.05 & 0.07 & 0.02 & 0.04 & 0.06 \\
\refrev{GPK-small} & 0.06 & 0.04 & 0.03 & 0.03 & 0.04 & 0.04 & 0.04 & 0.04 & 0.02 \\
GET & 0.05 & 0.05 & 0.05 & 0.03 & 0.05 & 0.04 & 0.08 & 0.06 & 0.06 \\
FID & 0.03 & 0.05 & 0.07 & 0.03 & 0.08 & 0.07 & 0.03 & 0.08 & 0.07 \\
KID & 0.02 & 0.04 & 0.05 & 0.05 & 0.04 & 0.05 & 0.04 & 0.04 & 0.05 \\
\refrev{Gaussian-kernel MMD} & 0.05 & 0.04 & 0.04 & 0.03 & 0.05 & 0.03 & 0.03 & 0.05 & 0.03 \\
MIND & 0.04 & 0.03 & 0.04 & 0.04 & 0.06 & 0.04 & 0.08 & 0.06 & 0.03 \\
C2ST & 0.01 & 0.05 & 0.04 & 0.01 & 0.07 & 0.10 & 0.01 & 0.00 & 0.05 \\
Precision & 0.03 & 0.04 & 0.06 & 0.03 & 0.02 & 0.05 & 0.02 & 0.02 & 0.05 \\
Recall & 0.00 & 0.01 & 0.05 & 0.03 & 0.03 & 0.04 & 0.02 & 0.04 & 0.03 \\
Density & 0.01 & 0.04 & 0.02 & 0.05 & 0.04 & 0.04 & 0.02 & 0.04 & 0.03 \\
Coverage & 0.03 & 0.02 & 0.05 & 0.05 & 0.03 & 0.03 & 0.01 & 0.03 & 0.01 \\
ECS ($T=.28$) & 0.04 & 0.02 & 0.03 & 0.05 & 0.07 & 0.05 & 0.04 & 0.07 & 0.07 \\
\refrev{Vendi Score} & 0.04 & 0.03 & 0.01 & 0.03 & 0.06 & 0.03 & 0.03 & 0.05 & 0.03 \\
PERMDISP & 0.07 & 0.02 & 0.06 & 0.06 & 0.05 & 0.07 & 0.06 & 0.05 & 0.06 \\
\bottomrule
\end{tabular}
\end{table}

\FloatBarrier

\end{document}